\documentclass[11pt]{article}
\usepackage{graphicx} 
\usepackage{bbm}
\usepackage[margin=1in]{geometry}
\usepackage{color}
\usepackage{amsfonts}
\usepackage{algorithm}
\usepackage{algorithmic}
\usepackage{latexsym, amssymb, amsmath, amscd, amsthm, amsxtra}
\usepackage{mathtools}
\usepackage{enumerate}
\usepackage[all]{xy}
\usepackage{mathrsfs}
\usepackage{fancyhdr}
\usepackage{listings}
\usepackage{hyperref}
\usepackage{enumitem}

\newcommand{\doublesymdiff}{%
  \mathbin{\text{\raisebox{0.20ex}{\scalebox{0.60}{$\triangle$}}\hspace{-0.70em}$\triangle$}}%
}

\def\11{\mathbbm{1}}

\def\ER{Erd\H{o}s-R\'enyi\ }

\newcommand{\Pb}{\mathbb P}
\newcommand{\Qb}{\mathbb Q}

\newtheorem{thm}{Theorem}[section]

\newtheorem{proposition}[thm]{Proposition}
\newtheorem{lemma}[thm]{Lemma}
\newtheorem{cor}[thm]{Corollary}
\newtheorem{defn}[thm]{Definition}

\newtheorem{remark}[thm]{Remark}
\newtheorem{conjecture}[thm]{Conjecture}
\newtheorem{question}[thm]{Question}

\numberwithin{equation}{section}



\hypersetup{colorlinks=true,linkcolor=blue}
\title{Computational Lower Bounds for Correlated Random Graphs via Algorithmic Contiguity}

\usepackage{authblk}
\author[1]{Zhangsong Li\thanks{Email: \textit{ramblerlzs@pku.edu.cn}. Partially supported by National Key R$\&$D program of China (Project No. 2023YFA1010103) and NSFC Key Program (Project No. 12231002).}}
\affil[1]{School of Mathematical Sciences, Peking University}

\date{}
\begin{document}
\maketitle

\begin{abstract}
    In this paper, assuming the low-degree conjecture, we provide evidence of computational hardness for two problems: (1) the (partial) matching recovery problem in the sparse correlated Erd\H{o}s-R\'enyi graphs $\mathcal G(n,q;\rho)$ when the edge-density $q=n^{-1+o(1)}$ and the correlation $\rho<\sqrt{\alpha}$ lies below the Otter's threshold, this resolves a remaining problem in \cite{DDL23+}; (2) the detection problem between a pair of correlated sparse stochastic block models $\mathcal S(n,\tfrac{\lambda}{n};k,\epsilon;s)$ and a pair of independent stochastic block models $\mathcal S(n,\tfrac{\lambda s}{n};k,\epsilon)$ when $\epsilon^2 \lambda s<1$ lies below the Kesten-Stigum (KS) threshold and $s<\sqrt{\alpha}$ lies below the Otter's threshold, this resolves a remaining problem in \cite{CDGL24+}.
    
    One of the main ingredient in our proof is to derive certain forms of \emph{algorithmic contiguity} between two probability measures based on bounds on their low-degree advantage. To be more precise, consider the high-dimensional hypothesis testing problem between two probability measures $\mathbb{P}$ and $\mathbb{Q}$ based on the sample $\mathsf Y$. We show that if the low-degree advantage $\mathsf{Adv}_{\leq D} \big( \frac{\mathrm{d}\mathbb{P}}{\mathrm{d}\mathbb{Q}} \big)=O(1)$, then (assuming the low-degree conjecture) there is no efficient algorithm $\mathcal A$ such that $\mathbb{Q}(\mathcal A(\mathsf Y)=0)=1-o(1)$ and $\mathbb{P}(\mathcal A(\mathsf Y)=1)=\Omega(1)$. This framework provides a useful tool for performing reductions between different inference tasks, without requiring a strengthened version of the low-degree conjecture as in \cite{MW23+, DHSS25+}.
\end{abstract}

\tableofcontents

\section{Introduction}{\label{sec:intro}}

Graph matching, also referred to as network alignment, is the problem of identifying a bijection between the vertex sets of two graphs that maximizes the number of common edges. When the two graphs are exactly isomorphic to each other, this problem reduces to the classical graph isomorphism problem, for which the best known algorithm runs in quasi-polynomial time \cite{Babai16}. In general, graph matching is an instance of the \emph{quadratic assignment problem} \cite{BCPP98}, which is known to be NP-hard to solve or even approximate \cite{MMS10}. Motivated by real-world applications (such as social network de-anonymization \cite{NS08} and computational biology \cite{SXB08}) as well as the need to understand the average-case computational complexity, recent research has focused on developing theoretical foundations and efficient algorithms for graph matching under statistical models. These models assume that the two graphs are randomly generated with correlated edges under a hidden vertex correspondence, and a canonical model among them is the following \emph{correlated random graph model}. For any integer $n$, denote by $\operatorname{U}= \operatorname{U}_n$ the set of unordered pairs $(i,j)$ with $1\le i\neq j\le n$. 

\begin{defn}[Correlated random graph model]{\label{def-correlated-random-graph}}
    Given an integer $n\ge 1$, for $(i,j) \in \operatorname{U}$ let $J_{i,j}$ and $K_{i,j}$ be independent Bernoulli variables with parameter $s$. In addition, let $\pi_*$ be an independent uniform permutation on $[n]=\{1,\dots,n\}$. Then, we define a triple of correlated random graphs $(G,A,B)$ as follows: first we generate $G$ independently with $\{ J_{i,j}, K_{i,j} \}$ and $\pi_*$ from a specific probability distribution over all graphs on $[n]$, and then (conditioned on $G$) we define for each $(i,j) \in \operatorname{U}$ that (note that we identify a graph with its adjacency matrix)
    \begin{align*}
        A_{i,j}=G_{i,j} J_{i,j} \,, \quad B_{i,j}=G_{i,j} K_{\pi_*^{-1}(i),\pi_*^{-1}(j)} \,.
    \end{align*}
    In short, we will subsample $A,B$ from $G$ with subsampling probability $s$ and then permute the vertices of $B$ by a uniform permutation. 
\end{defn}

Of particular interest in our paper are the following two special cases, namely the correlated \ER model and the correlated stochastic block model (SBM).

\begin{defn}[Correlated \ER graph model]{\label{def-correlated-ER-graph}}
    Given an integer $n\ge 1$ and two parameters $p,s\in (0,1)$, we generate a triple of correlated random graphs $(G,A,B)$ such that we first generate $G$ according to an \ER graph distribution $\mathcal G(n,p)$ (i.e., for each $(i,j) \in \operatorname{U}$ we connect $(i,j)$ in $G$ independently with probability $p$), and then generate $(A,B)$ from $G$ according to Definition~\ref{def-correlated-random-graph}. For ease of presentation, we shall reparameterize such that $q=ps$ and $\rho=\frac{s(1-p)}{1-ps}$ respectively. We will denote the marginal law of $(A,B)$ as $\mathcal G(n,q;\rho)$. 
\end{defn}

\begin{defn}[Stochastic block model] {\label{def-SBM}}
    Given an integer $n\ge 1$ and three parameters $k \in \mathbb{N}, \lambda>0, \epsilon \in (0,1)$, we define a random graph $G$ as follows: (1) sample a labeling $\sigma_* \in [k]^{n} = \{ 1,\ldots,k \}^{n}$ uniformly at random; (2) for every distinct pair $(i,j) \in \operatorname{U}_n$, we let $G_{i,j}$ be an independent Bernoulli variable such that $G_{i,j} = 1$ (which represents that there is an undirected edge between $i$ and $j$) with probability $\frac{(1+(k-1)\epsilon)\lambda}{n}$ if $\sigma_*(i) = \sigma_*(j)$ and with probability $\frac{(1-\epsilon)\lambda}{n}$ if $\sigma_*(i) \neq \sigma_*(j)$. In this case, we say that $G$ is sampled from a stochastic block model $\mathcal S(n,\tfrac{\lambda}{n};k,\epsilon)$.
\end{defn}

\begin{defn}[Correlated stochastic block model] {\label{def-correlated-SBM}}
    Given an integer $n\ge 1$ and four parameters $k \in \mathbb{N}, \lambda>0, \epsilon,s \in (0,1)$, we define a triple of correlated random graphs $(G,A,B)$ as follows: we first sample $G$ according to the law of a stochastic block model $\mathcal S(n,\tfrac{\lambda}{n}; k,\epsilon)$ and then generate $(A,B)$ from $G$ according to Definition~\ref{def-correlated-random-graph}. We will denote the marginal law of $(A,B)$ as $\mathcal S(n,\tfrac{\lambda}{n}; k,\epsilon;s)$.
\end{defn}

Two fundamental problems in the study of correlated random graph model are as follows: (1) the \emph{detection} problem, which involves determining whether a given pair of graphs $(A,B)$ is sampled from a pair of correlated random graphs or from a pair of independent random graphs; (2) the \emph{matching} problem, which focuses on recovering the latent matching $\pi_*$ from a sample $(A,B)$ from the distribution of correlated random graphs. In recent years, significant progress has been made in understanding these problems for both the correlated \ER model and correlated stochastic block models. Through the collective efforts of the community, the information-theoretic thresholds for detection and matching have been fully characterized for correlated \ER model \cite{CK17+, HM23, WXY22, WXY23, GML21, DD23a, DD23b, Du25+} and partially characterized for correlated SBMs \cite{RS21, GRS22}. Additionally, various efficient detecting and matching algorithms have been developed with performance guarantees \cite{DCKG19, BCL+19, DMWX21, FMWX23a, FMWX23b, GM20, GML24, GMS24, MRT21, MRT23, MWXY24, MWXY23, DL22+, DL23+, CR24, CDGL24+, Li24+}. Notably, like many other inference tasks in high-dimensional statistics \cite{ZK16, RSS19, KWB22, Gamarnik21}, these problems appear to exhibit \emph{information-computation gaps}. Specifically, for certain ranges of the correlation strength, detection or matching is information theoretically possible but no efficient algorithm is known to achieve these tasks. We now focus on the algorithmic side of these problems as they are more relevant to our work. Indeed, it has been shown that many inference tasks in the correlated random graph models exhibits sharp algorithmic phase transitions, as we elaborate below:
\begin{itemize}
    \item For the detection problem between a pair of correlated \ER models $\mathcal G(n,q;\rho)$ and two independent \ER models $\mathcal G(n,q)$, we focus on the sparse regime where $q=n^{-1+o(1)}$. In this regime, on the one hand, it was shown in \cite{MWXY24} that when $\rho>\sqrt{\alpha}$ where $\alpha\approx 0.338$ is the Otter's constant \cite{Otter48}, there is an efficient algorithm that strongly distinguish these two models; on the other hand, it was shown in \cite{DDL23+} that when $\rho<\sqrt{\alpha}$ there is evidence suggesting that all algorithms based on \emph{low-degree polynomials} fail to strongly distinguish these two models.
    \item For the detection problem between a pair of correlated SBMs $\mathcal S(n,\frac{\lambda}{n};k,\epsilon;s)$ and two independent \ER graphs $\mathcal G(n,\frac{\lambda s}{n})$, we focus on the constant degree regime where $\lambda=O(1)$. In this regime, on the one hand, it was shown in \cite{CDGL24+} that when $s>\min\{ \sqrt{\alpha}, \tfrac{1}{\epsilon^2 \lambda} \}$ where $\alpha\approx 0.338$ is the Otter's constant and $\tfrac{1}{\epsilon^2 \lambda}$ is the Kesten-Stigum threshold \cite{KS66}, there is an efficient algorithm that strongly distinguish these two models; on the other hand, it was also shown in \cite{CDGL24+} that when $s<\min\{ \sqrt{\alpha}, \tfrac{1}{\epsilon^2 \lambda} \}$ there is evidence suggesting that all algorithms based on \emph{low-degree polynomials} fail to strongly distinguish these two models.
\end{itemize}

The lower bound in the aforementioned results explored inherent computational barriers from the perspective of the \emph{low-degree polynomial framework}. Indeed, it has been proved that the class of low-degree polynomial algorithms is a useful proxy for computationally efficient algorithms, in the sense that the best-known polynomial-time algorithms for a wide variety of high-dimensional inference problems are captured by the low-degree class; see e.g. \cite{HKP+17, Hopkins18, SW22, KWB22}. However, these aforementioned results suffer from two significant limitations, which we now discuss. Firstly, the aforementioned result cannot provide the evidence that \emph{partial matching recovery} (that is, recover a positive fraction of the coordinates of the latent matching $\pi_*$) is impossible by efficient algorithms although efficient detection has already been ruled out. Secondly, (in the case of correlated SBM) they are only able to establish the computation threshold on the detection problem between the correlated model and a pair of independent \ER graphs. These limitations motivate two natural questions:

\begin{question}{\label{natural-questions}}
    \begin{enumerate}
        \item[(1)] Can we provide the evidence that partial matching recovery is impossible (by efficient algorithms) in the same parameter regime where there is evidence suggesting that detection is impossible (by efficient algorithms)?
        \item[(2)] What can we say about the (arguably more natural) detection problem between a pair of correlated SBMs $\mathcal S(n,\tfrac{\lambda}{n};k,\epsilon;s)$ and a pair of independent SBMs $\mathcal S(n,\tfrac{\lambda s}{n};k,\epsilon)$?
    \end{enumerate}
\end{question}

The aim of this paper is to (partially) answer these two problems. Our main result can be informally summarized as follows:

\begin{thm}[Informal]{\label{MAIN-THM-informal}}
    Assuming the low-degree conjecture (see Conjecture~\ref{conj-low-deg} for its precise meaning), we have the following results. 
    \begin{enumerate}
        \item[(1)] For the correlated \ER model $\mathcal G(n,q,\rho)$, when $q=n^{-1+o(1)}$ and $\rho<\sqrt{\alpha}$ it is impossible to recover a positive fraction of the coordinates of $\pi_*$ with constant probability by efficient algorithms.
        \item[(2)] For the correlated stochastic block models $\mathcal S(n,\tfrac{\lambda}{n};k,\epsilon;s)$, when $\lambda=O(1)$ and $s<\min\{ \sqrt{\alpha}, \tfrac{1}{\epsilon^2 \lambda} \}$ it is impossible to recover a positive fraction of the coordinates of $\pi_*$ with probability tending to $1$ as $n\to\infty$ by efficient algorithms.
        \item[(3)] For the correlated stochastic block models $\mathcal S(n,\tfrac{\lambda}{n};k,\epsilon;s)$, when $s<\min\{ \sqrt{\alpha}, \tfrac{1}{\epsilon^2 \lambda} \}$ it is impossible to strongly distinguish this model and a pair of independent SBMs $\mathcal S(n,\tfrac{\lambda s}{n};k,\epsilon)$ by efficient algorithms, provided that the average degree $\lambda s$ is sufficiently large.
    \end{enumerate}
\end{thm}

\begin{remark}
    Note that when $\rho>\sqrt{\alpha}$, the results in \cite{MWXY23, GM20, GML24, GMS24} show that there exists an efficient algorithm that achieves partial recovery of $\pi_*$ in a pair of correlated \ER models with probability tending to $1$ as $n\to\infty$. Thus, Item~(1) in Theorem~\ref{MAIN-THM-informal} is tight in some sense and the algorithmic partial recovery threshold is indeed given by $\rho=\sqrt{\alpha}$. 
    
    Similarly, the results in \cite{MWXY23, GM20, GML24, GMS24} naturally extend to show that there exists an efficient algorithm that achieves partial recovery of $\pi_*$ in a pair of correlated SBMs when $s>\sqrt{\alpha}$ with probability tending to $1$ as $n\to\infty$. Thus, Item~(2) in Theorem~\ref{MAIN-THM-informal} is also tight in the subcritical regime (i.e., when marginally two graphs are below the KS threshold).
\end{remark} 

\begin{remark}
    From Item~(3) in Theorem~\ref{MAIN-THM-informal} we see that for a pair of correlated SBMs $(A,B)$, when marginally both $A$ and $B$ are below the KS threshold (i.e., when $\epsilon^2 \lambda s<1$), there is evidence suggesting that no efficient algorithm can strongly distinguish $(A,B)$ from a pair of independent stochastic block models when the correlation $s<\sqrt{\alpha}$. Since the result in \cite{MWXY24} extends naturally to the case of stochastic block models which provides an efficient algorithm that strongly distinguish these two models when $s>\sqrt{\alpha}$, we see that in the subcritical regime (i.e., when marginally both $A$ and $B$ are below the KS threshold) the algorithmic correlation detection threshold is given by $\sqrt{\alpha}$. On the contrary, in the supercritical regime where $A$ and $B$ are above the KS threshold (i.e, when $\epsilon^2 \lambda s>1$), we believe that the algorithmic correlation detection threshold should be strict lower than $\sqrt{\alpha}$. This was further supported by a recent work \cite{CDGL25+}, which has shown that in the case of $k=2$, strong detection between $\Pb_n$ and $\widetilde{\Qb}_n$ is achievable when $\epsilon^2 \lambda s>1$ and $s\geq \sqrt{\alpha}-\delta$ where $\delta$ is a sufficiently small constant. However, rigorous analysis in the general case seems of substantial challenge and we leave it for future work.
\end{remark}

\subsection{Key challenges and innovations}{\label{subsec:innovation}}

In this subsection, we briefly discuss our approach of proving Theorem~\ref{MAIN-THM-informal} and some conceptual innovations behind it. Our idea can be informally summarized as follows:
\begin{enumerate}
    \item[(1)] As for Items~(1) and (2), for simplicity we take correlated \ER model for example. Denote $\Pb$ to be the law of $\mathcal G(n,q,\rho)$ and $\Qb$ to be the law of two independent $\mathcal G(n,q)$. We will argue by contradiction and assume that there is a partial recovery algorithm. Then we show that we can use this algorithm to efficiently construct a family of statistics $\{ g_i: 1 \leq i \leq n \}$ such that $g_i$ approximates $\mathbf 1_{ \{ \pi_*(1)=i \} }$ under $\Pb$ in a certain sense. 
    \item[(2)] We will show that the low-degree advantage between $\Pb(\cdot \mid \pi_*(1)=i)$ and $\Qb$ is bounded by an absolute constant. Then, from the standard low-degree conjecture (see Conjecture~\ref{conj-low-deg} for details) these two measures cannot be strongly distinguished by efficient algorithms. Thus, since $g_i$ is ``of positive constant order'' under the measure $\Pb(\cdot \mid \pi_*(i)=1)$ (as it should approximate $\mathbf 1_{ \{ \pi_*(1)=i \} }$ in some sense) we expect that $g_i$ should also be ``of positive constant order'' under $\Qb$.
    \item[(3)] Define $g=g_1+\ldots+g_n$. We then have a statistic $g$ that can be efficiently computed, and (I) under $\Qb$ we expect that $g$ is ``large'' as it is the sum of $n$ ``of positive constant order'' terms; (II) under $\Pb$ we expect that $g$ is ``not large'' as each $g_i$ should approximate $\mathbf 1_{ \{ \pi_*(1)=i \} }$ in some sense. Thus, the statistics $g$ should achieves detection in a certain sense, which violates our assumption.
    \item[(4)] As for Item~(3), denote $\Pb$ to be the law of correlated SBMs and $\Qb$ of independent SBMs. Also denote $\widetilde{\Qb}$ to be the law of independent \ER graphs. As it was already shown in \cite{CDGL24+} that $\Pb$ and $\widetilde{\Qb}$ cannot be strongly distinguished, (non-rigorously speaking) we only need to show that $\Qb$ and $\widetilde{\Qb}$ are also ``indistinguishable'' in some sense.
\end{enumerate}

However, there are certain obstacles when implementing the above ideas, as we shall discuss below. In Step~(2) we need to ``transfer'' the behavior of a statistic $g_i$ under $\Pb(\cdot \mid \pi_*(1)=i)$ to its behavior under $\Qb$. If the low-degree advantage between $\Pb(\cdot \mid \pi_*(1)=i)$ and $\Qb$ is bounded by $1+o(1)$, then the standard low-degree conjecture implies that we cannot distinguish between $\Pb(\cdot \mid \pi_*(1)=i)$ and $\Qb$ better than random by efficient algorithms. Consequently,  we expect the behavior of $g_i$ to be ``almost identical'' under both $\Pb(\cdot \mid \pi_*(1)=i)$ and $\Qb$ (as $g_i$ can be efficiently computed). However, in our case the low-degree advantage is just bounded by a large (but fixed) constant. This weaker condition means that the standard low-degree conjecture only rules out the possibility to distinguish these two measures efficiently with vanishing errors, leaving room for non-negligible distinctions. Similar issues also arise in Step~(4), in which simply showing the low-degree advantage between $\Pb$ and $\widetilde{\Qb}$ and the low-degree advantage between $\Qb$ and $\widetilde{\Qb}$ is bounded is not strong enough for our goal. 

To address these issues, one of the main conceptual contributions in our work is to give a more refined characterization of the limitations of efficient algorithms when the low-degree advantage is $O(1)$. Specifically, we show that (assuming low-degree conjecture) bounded low-degree advantage does not only exclude all algorithms that strongly distinguish $\Pb$ and $\Qb$, but also suggest a certain kind of \emph{algorithmic contiguity}. To be more precise, if the low-degree advantage $\mathsf{Adv}_{\leq D}\big( \tfrac{\mathrm{d}\Pb}{\mathrm{d}\Qb} \big)$ is bounded, then there is no efficient algorithm $\mathcal A$ such that $\Qb(\mathcal A=0)=1-o(1)$ and $\Pb(\mathcal A=1)=\Omega(1)$. This framework allows us to transfer the behavior of a efficiently computable statistics $g$ under different probability measures more easily. For example, if we know that $\Pb(g \geq c) \geq \Omega(1)$ it immediately holds that $\Qb(g \geq c) \geq \Omega(1)$.


\subsection{Comparison to concurrent work}{\label{subsec:concurrent-work}}

In a concurrent work \cite{DHSS25+}, the authors investigate the computational hardness of the weak recovery problem in the stochastic block model $\mathcal S(n,\tfrac{\lambda}{n};k,\epsilon)$ when $k=n^{o(1)}$ and $\epsilon^2 \lambda<1$ lies below the KS threshold. Notably, their proof is also based on a ``recovery-to-detection reduction'' approach, which argues that if there exists an efficient algorithm that achieves weak recovery, then we can use this algorithm to efficiently construct a statistic for detection (in a certain sense). Their results relies on a \emph{strengthening} of the low-degree conjecture, which is inspired by the work \cite{MW23+}. 

However, the reduction technique employed in this paper differs significantly from that of \cite{DHSS25+}. The authors of \cite{DHSS25+} construct their statistic for detection using a \emph{correlation preserving projection} technique established in \cite{HS17}. This technique, which is designed specifically for weak recovery in block models, enables them to ``regularize'' the statistics and thus bound its moments under the null hypothesis. In contrast, our construction is more straightforward and only requires the original low-degree conjecture as we leverage the universal framework described in Section~\ref{subsec:innovation} to ``transfer'' the behavior of statistics across different probability measures. It seems that in comparison to \cite{DHSS25+}, our approach is easier to implement, less dependent on problem-specific details, and potentially applicable to a broader range of problems.

\subsection{Notation}{\label{subsec:notations}}

In this subsection, we record a list of notations that we shall use throughout the paper. Denote $\mathfrak{S}_n$ the set of permutations over $[n]$ and denote $\mu$ the uniform distribution on $\mathfrak S_n$. In addition, denote $\nu$ to be the uniform distribution on $[k]^n$. We will use the following notation conventions on graphs.
\begin{itemize}
    \item {\em Labeled graphs}. Denote by $\mathcal{K}_n$ the complete graph with vertex set $[n]$ and edge set $\operatorname{U}_n$. For any graph $H$, let $V(H)$ denote the vertex set of $H$ and let $E(H)$ denote the edge set of $H$. We say $H$ is a subgraph of $G$, denoted by $H\subset G$, if $V(H) \subset V(G)$ and $E(H) \subset E(G)$. Define the excess of the graph $\tau(H)=|E(H)|-|V(H)|$.
    \item {\em Isolated vertices}. For $u \in V(H)$, we say $u$ is an isolated vertex of $H$, if there is no edge in $E(H)$ incident to $u$. Denote $\mathcal I(H)$ the set of isolated vertices of $H$. For two graphs $H,S$, we denote $H \ltimes S$ if $H \subset S$ and $\mathcal I(S) \subset \mathcal I(H)$, and we denote $H \Subset S$ if $H \subset S$ and $\mathcal I(H)=\emptyset$. For any graph $H \subset \mathcal K_n$, let $\widetilde H$ be the subgraph of $H$ induced by $V(H) \setminus \mathcal I(H)$. 
    \item {\em Graph intersections and unions}. For $H,S \subset \mathcal{K}_n$, denote by $H \cap S$ the graph with vertex set given by $V(H) \cap V(S)$ and edge set given by $E(H)\cap E(S)$. Denote by $S \cup H$ the graph with vertex set given by $V(H) \cup V(S)$ and edge set $E(H) \cup E(S)$. In addition, denote by $S\Cap H$, $S \Cup H$ and $S\doublesymdiff H$ the graph induced by the edge set $E(S)\cap E(H)$, $E(S)\cup E(H)$ and $ E(S)\triangle E(H)$, respectively (in particular, these induced graphs have no isolated points).
    \item {\em Paths.} We say a subgraph $H \subset \mathcal K_n$ is a path with endpoints $u,v$ (possibly with $u=v$), if there exist distinct $w_1, \ldots, w_m \neq u,v$ such that $V(H)=\{ u,v,w_1,\ldots,w_m \}$ and $E(H)=\{ (u,w_1), (w_1,w_2) \ldots, (w_m,v) \}$. We say $H$ is a simple path if its endpoints $u \neq v$. Denote $\operatorname{EndP}(P)$ as the set of endpoints of a path $P$.
    \item {\em Cycles and independent cycles.} We say a subgraph $H$ is an $m$-cycle if $V(H)=\{ v_1, \ldots, v_m \}$ and $E(H)=\{ (v_1,v_2), \ldots, (v_{m-1},v_m), (v_m,v_1) \}$. For a subgraph $K \subset H$, we say $K$ is an independent $m$-cycle of $H$, if $K$ is an $m$-cycle and no edge in $E(H)\setminus E(K)$ is incident to $V(K)$. Denote by $\mathtt{C}_m(H)$ the set of $m$-cycles of $H$ and denote by $\mathcal{C}_m(H)$ the set of independent $m$-cycles of $H$. For $H\subset S$, we define $\mathfrak{C}_{m}(S,H)$ to be the set of independent $m$-cycles in $S$ whose vertex set is disjoint from $V(H)$. Define $\mathfrak{C}(S,H)=\cup_{m\geq 3} \mathfrak{C}_m(S,H)$.
    \item {\em Leaves.} A vertex $u \in V(H)$ is called a leaf of $H$, if the degree of $u$ in $H$ is $1$; denote $\mathcal L(H)$ as the set of leaves of $H$. 
    \item {\em Graph isomorphisms and unlabeled graphs.} Two graphs $H$ and $H'$ are isomorphic, denoted by $H\cong H'$, if there exists a bijection $\pi:V(H) \to V(H')$ such that $(\pi(u),\pi(v)) \in E(H')$ if and only if $(u,v)\in E(H)$. Denote by $\mathcal H$ the isomorphism class of graphs; it is customary to refer to these isomorphic classes as unlabeled graphs. Let $\operatorname{Aut}(H)$ be the number of automorphisms of $H$ (graph isomorphisms to itself). 
\end{itemize}

For two real numbers $a$ and $b$, we let $a \vee b = \max \{ a,b \}$ and $a \wedge b = \min \{ a,b \}$. We use standard asymptotic notations: for two sequences $a_n$ and $b_n$ of positive numbers, we write $a_n = O(b_n)$, if $a_n<Cb_n$ for an absolute constant $C$ and for all $n$ (similarly we use the notation $O_h$ is the constant $C$ is not absolute but depends only on $h$); we write $a_n = \Omega(b_n)$, if $b_n = O(a_n)$; we write $a_n = \Theta(b_n)$, if $a_n =O(b_n)$ and $a_n = \Omega(b_n)$; we write $a_n = o(b_n)$ or $b_n = \omega(a_n)$, if $a_n/b_n \to 0$ as $n \to \infty$. In addition, we write $a_n \circeq b_n$ if $a_n = [1+o(1)] b_n$. For a set $\mathsf A$, we will use both $\# \mathsf A$ and $|\mathsf A|$ to denote its cardinality. For two probability measures $\mathbb P$ and $\mathbb Q$, we denote the total variation distance between them by $\operatorname{TV}(\mathbb P,\mathbb Q)$.

\subsection{Organization of this paper}

The rest part of this paper is organized as follows: in Section~\ref{sec:revise-low-deg-conj} we state the precise framework of low-degree polynomials and how we relate it to the notion of algorithmic contiguity (see Theorem~\ref{main-thm-alg-contiguity}). In Section~\ref{sec:application-1} we use this framework to deduce the hardness of partial matching in correlated \ER models and correlated SBMs, thus verifying Items~(1) and (2) in Theorem~\ref{MAIN-THM-informal} (see Theorem~\ref{main-thm-cor-ER}). In Section~\ref{sec:application-2} we again use this framework to deduce the hardness of testing correlated SBMs against independent SBMs, thus verifying Item~(3) in Theorem~\ref{MAIN-THM-informal} (see Corollary~\ref{cor-correlated-SBM}). Several auxiliary proofs are moved to the appendix to ensure a smooth flow of presentation.

\section{Low-degree framework and algorithmic contiguity}{\label{sec:revise-low-deg-conj}}

The low-degree polynomial framework first emerged from the works of \cite{BHK+19, HS17, HKP+17, Hopkins18} and has since been refined and extended in various directions. It has found applications across a broad spectrum of problems, including detection tasks such as planted clique, planted dense subgraph, community detection, sparse PCA, and tensor PCA (see \cite{HS17, HKP+17, Hopkins18, KWB22, SW22, DMW23+, BKW20, DKW+22, MW22+, DDL23+, KMW24}), optimization problems like finding maximal independent sets in sparse random graphs \cite{GJW20, Wein22}, and constraint satisfaction problems such as random $k$-SAT \cite{BH22}; see also the survey \cite{KWB22}. Additionally, it is conjectured in \cite{Hopkins18} that the failure of degree-$D$ polynomials implies the failure of all ``robust'' algorithms with running time $n^{\widetilde{O}(D)}$ (here $\widetilde{O}$ means having at most this order up to a $\operatorname{poly} \log n$ factor). However, to prevent the readers from being overly optimistic for this conjecture, we point out a recent work \cite{BHJK25} that finds a counterexample for this conjecture. We therefore clarify that the low-degree framework is expected to be optimal for a certain, yet imprecisely defined, class of ``high-dimensional'' problems. Despite these important caveats, we still believe that analyzing low-degree polynomials remains highly meaningful for our setting, as it provides a benchmark for robust algorithmic performance. We refer the reader to the survey \cite{Wein25+} for a more detailed discussion of these subtleties. In the remaining of this paper, we will focus on applying this framework in the context of high-dimensional hypothesis testing problems.

To be more precise, consider the hypothesis testing problem between two probability measures $\Pb$ and $\Qb$ based on the sample $\mathsf Y \in \mathbb R^{N}$. We will be especially interested in asymptotic settings where $N=N_n, \Qb=\Qb_n, \Pb=\Pb_n, \mathsf Y=\mathsf Y_n$ scale with $n$ as $n \to \infty$ in some prescribed way. The standard low-degree polynomial framework primarily focus on the following notions on strong and weak detection.
\begin{defn}[Strong/weak detection]{\label{def-strong-weak-detection}}
    We say an algorithm $\mathcal A$ that takes $\mathsf Y$ as input and outputs either $0$ or $1$ achieves
    \begin{itemize}
        \item {\bf strong detection}, if the sum of type-I and type-II errors $\Qb(\mathcal A(\mathsf Y)=1) + \Pb(\mathcal A(\mathsf Y)=0)$ tends to $0$ as $n\to \infty$. 
        \item {\bf weak detection}, if the sum of type-I and type-II errors is uniformly bounded above by $1-\epsilon$ for some fixed $\epsilon>0$.
    \end{itemize}
\end{defn}

Roughly speaking, the low-degree polynomial approach focuses on understanding the capabilities and limitations of algorithms that can be expressed as low-degree polynomial functions of the input variables (in our case, the entries of $\mathsf Y$). To be more precise, let $\mathcal P_D=\mathcal P_{n,D}$ denote the set of polynomials from $\mathbb R^{N}$ to $\mathbb R$ with degree no more than $D$. With a slight abuse of notation, we will often refer to ``a polynomial'' to mean a sequence of polynomials $f=f_n \in \mathcal P_{n,D}$, where each $f_n$ corresponds to problem size $n$; the degree $D=D_n$ of such a polynomial may scale with $n$. As suggested by \cite{Hopkins18}, the key quantity is the low-degree advantage
\begin{equation}{\label{eq-def-low-deg-Adv}}
    \mathsf{Adv}_{\leq D}(\Pb,\Qb)\Big( \frac{ \mathrm{d}\Pb }{ \mathrm{d}\Qb } \Big) := \sup_{f \in \mathcal P_{D}} \frac{ \mathbb E_{\Pb}[f] }{ \sqrt{ \mathbb E_{\Qb}[f^2] } } \,.
\end{equation}
Note that if we denote the likelihood ratio $L(\mathsf Y)=\frac{ \mathrm{d}\Pb }{ \mathrm{d}\Qb }(\mathsf Y)$, it is well-known (see, e.g., \cite{Hopkins18}) the right hand side equals the $L_2$-norm of the projection of $L(\mathsf Y)$ into the subspace spanned by all polynomials of degree bounded by $D$ (the norm is induced by natural inner product under $\Qb$) and thus we characterize it with $\frac{ \mathrm{d}\Pb }{ \mathrm{d}\Qb }$.
The low degree conjecture, proposed in \cite{Hopkins18}, can be summarized as follows.
\begin{conjecture}[Low-degree conjecture]{\label{conj-low-deg}}
    For ``natural'' high-dimensional hypothesis testing problems between $\Pb$ and $\Qb$, the following statements hold.
    \begin{enumerate}
        \item[(1)] If $\mathsf{Adv}_{\leq D}\big( \tfrac{\mathrm{d}\Pb'}{\mathrm{d}\Qb'} \big)=O(1)$ as $n\to\infty$ for some $\Pb',\Qb'$ such that $\operatorname{TV}(\Pb,\Pb'), \operatorname{TV}(\Qb,\Qb')=o(1)$, then there exists a constant $C$ such that no algorithm with running time $n^{D/(\log n)^C}$ that achieves strong detection between $\Pb$ and $\Qb$. 
        \item[(2)] If $\mathsf{Adv}_{\leq D}\big( \tfrac{\mathrm{d}\Pb'}{\mathrm{d}\Qb'} \big)=1+o(1)$ as $n\to\infty$ for some $\Pb',\Qb'$ such that $\operatorname{TV}(\Pb,\Pb'), \operatorname{TV}(\Qb,\Qb')=o(1)$, then there exists a constant $C$ such that no algorithm with running time $n^{D/(\log n)^C}$ that achieves weak detection between $\Pb$ and $\Qb$. 
    \end{enumerate}
\end{conjecture}

Note that in Conjecture~\ref{conj-low-deg} we are allowed to replace $\Pb,\Qb$ with some $\Pb',\Qb'$ that are \emph{statistically indistinguishable} with $\Pb,\Qb$. This modification enables us to avoid some situations where the low-degree advantage explodes due to some ``rare events'' (see \cite{BAH+22, DDL23+, DMW23+} for example). Motivated by \cite[Hypothesis~2.1.5 and Conjecture~2.2.4]{Hopkins18} as well as the fact that low-degree polynomials capture the best known algorithms for a wide variety of statistical inference tasks, Conjecture~\ref{conj-low-deg} appears to hold for distributions of a specific form that frequently arises in high-dimensional statistics. For further discussion on what types of distributions are suitable for this framework, we refer readers to \cite{Hopkins18, KWB22, Kunisky21, ZSWB22}. In addition, we point out that although in most applications (and in the statement of \cite[Hypothesis 2.1.5]{Hopkins18}) it is typically to take $\Qb$ to be a ``null'' measure and $\Pb$ to be a ``planted'' measure (which makes \eqref{eq-def-low-deg-Adv} more tractable), several recent works \cite{RSWY23, KVWX23} showed that this framework might also be applicable for many ``planted-versus-planted'' problems. Nevertheless, in this paper, we adopt a more conservative view and will explicitly indicate whenever $\Qb$ is treated as a planted measure.

The framework in Conjecture~\ref{conj-low-deg} provides a useful tool for probing the computational feasibility of strong or weak detection. However, as discussed in Section~\ref{subsec:innovation}, it turns out that the failure of strong detection is not enough in our cases, especially when we hope to perform some reductions between statistical models in a regime where weak (but not strong) detection is possible. Thus, in this regime, we aim to characterize a stronger framework that rules out all \emph{one-sided test}. This motivates us to propose the following notion of \emph{algorithmic contiguity}.

\begin{defn}{\label{def-alg-contiguity}}
    For ``natural'' high-dimensional hypothesis testing problems between $\Pb$ and $\Qb$, we say an algorithm $\mathcal A$ that takes $\mathsf Y$ as input and outputs either $0$ or $1$ is a $\Qb$-based one-sided test, if
    \begin{equation}{\label{eq-def-one-sided-test}}
        \Qb(\mathcal A(\mathsf Y)=0) = 1-o(1) \mbox{ and } \Pb(\mathcal A(\mathsf Y)=1) = \Omega(1) \,.
    \end{equation}
    We say that $\Pb$ is time-$D$ algorithmic contiguous with respect to $\Qb$, denoted as $\Pb \lhd_{\leq D} \Qb$, if no $\Qb$-based one-sided testing algorithm runs in time $n^{D}$. We say that $\Qb$ and $\Pb$ are degree-$D$ algorithmic mutually contiguous, denoted as $\Qb\bowtie_{\leq D}\Pb$, if both $\Qb\lhd_{\leq D}\Pb$ and $\Pb\lhd_{\leq D}\Qb$ hold.
\end{defn}

Recall that in probability theory we say a sequence of probability measure $\Pb=\Pb_n$ is contiguous with respect to $\Qb=\Qb_n$, if for all sequence of events $\{ A_n \}$ we have $\Qb_n(A_n)\to 0$ implies that $\Pb_n(A_n)\to 0$. Thus, our definition can be regarded as the generalization of contiguity in algorithmic view. Our main result in this section can be stated as follows.
\begin{thm}{\label{main-thm-alg-contiguity}}
    Assuming the Conjecture~\ref{conj-low-deg}, for the high-dimensional hypothesis testing problem between $\Pb$ and $\Qb$, if $\mathsf{Adv}_{\leq D}\big( \tfrac{\mathrm{d}\Pb'}{\mathrm{d}\Qb'} \big)=O(1)$ for some $\Pb',\Qb'$ such that $\operatorname{TV}(\Pb,\Pb')=o(1)$ and $\operatorname{TV}(\Qb,\Qb')=o(1)$, then we have $\Pb\lhd_{\leq D/(\log n)^{C}}\Qb$ for some constant $C$.
\end{thm}

The goal of this section is devoted to the proof of Theorem~\ref{main-thm-alg-contiguity}. We first briefly explain our proof ideas. Suppose on the contrary that there is an efficient algorithm $\mathcal A$ that takes $\mathsf Y$ as input and outputs either $0$ or $1$ with
\begin{equation}{\label{eq-contradict-assumption}}
    \Pb(\mathcal A(\mathsf Y)=1)=\Omega(1) \mbox{ and } \Qb(\mathcal A(\mathsf Y)=0)=1-o(1)  \,.
\end{equation}
Since $\operatorname{TV}(\Qb,\Qb'), \operatorname{TV}(\Pb,\Pb')=o(1)$, we also have
\begin{equation}{\label{eq-contradict-assumption-relax}}
    \Pb'(\mathcal A(\mathsf Y)=1)=\Omega(1) \mbox{ and } \Qb'(\mathcal A(\mathsf Y)=0)=1-\epsilon \mbox{ for some } \epsilon=\epsilon_n\to 0 \,.
\end{equation}
Define
\begin{equation}{\label{eq-def-M-n}}
    M=M_n=\min\big\{ n, \lceil \epsilon_n^{-\frac{1}{2}} \rceil \big\} \to \infty \,.
\end{equation}
The crux of our argument is to consider the following \emph{hidden informative sample} problem.

\begin{defn}{\label{def-hidden-sample}}
    Consider the following hypothesis testing problem: we need to determine whether a sample $(\mathsf Y_1,\mathsf Y_2,\ldots, \mathsf Y_M)$ where each $\mathsf Y_i \in \mathbb R^N$ is generated by
    \begin{itemize}
        \item $\overline{\mathcal H}_0$: we let $\mathsf Y_1,\ldots,\mathsf Y_M$ to be independently sampled from $\Qb'$.
        \item $\overline{\mathcal H}_1$: we first sample $\kappa \in \{ 1,\ldots,M \}$ uniformly at random, and (conditioned on the value of $\kappa$) we let $\mathsf Y_1,\ldots,\mathsf Y_M$ are independent samples with $\mathsf Y_\kappa$ generated from $\Pb'$ and $\{ \mathsf Y_j: j \neq \kappa \}$ generated from $\Qb'$.
    \end{itemize}
    In addition, denote $\overline{\Pb}$ and $\overline{\Qb}$ to be the law of $(\mathsf Y_1,\ldots,\mathsf Y_M)$ under $\overline{\mathcal H}_1$ and $\overline{\mathcal H}_0$, respectively. 
\end{defn}

Now assuming that \eqref{eq-contradict-assumption-relax} holds, we see that
\begin{align}
    & \overline{\Qb}\Big( \big( \mathcal A(\mathsf Y_1), \ldots, \mathcal A(\mathsf Y_M) \big) = (0,\ldots,0) \Big) \geq 1-M\epsilon \overset{\eqref{eq-def-M-n}}{=} 1-o(1) \,; \label{eq-behavior-Qb-hidden-sample} \\
    & \overline{\Pb}\Big( \big( \mathcal A(\mathsf Y_1), \ldots, \mathcal A(\mathsf Y_M) \big) \neq (0,\ldots,0) \Big) \geq \Omega(1) \,. \label{eq-behavior-Pb-hidden-sample}
\end{align}
Thus, there is an efficient algorithm that achieves weak detection between $\overline{\Pb}$ and $\overline{\Qb}$. Our next result, however, shows the low-degree advantage $\mathsf{Adv}_{\leq D}\big( \tfrac{ \mathrm{d}\overline{\Pb} }{ \mathrm{d}\overline{\Qb} } \big)$ is bounded by $1+o(1)$.

\begin{lemma}{\label{lem-low-deg-hardness-hidden-sample}}
    If $\mathsf{Adv}_{\leq D}\big( \tfrac{ \mathrm{d}\Pb' }{ \mathrm{d}\Qb' } \big)=O(1)$, then $\mathsf{Adv}_{\leq D}\big( \tfrac{ \mathrm{d}\overline{\Pb} }{ \mathrm{d}\overline{\Qb} } \big)=1+o(1)$.
\end{lemma}
\begin{proof}
    Note that
    \begin{align}
        \frac{ \mathrm{d}\overline{\Pb} }{ \mathrm{d}\overline{\Qb} } (\mathsf Y_1,\ldots,\mathsf Y_M) = \frac{1}{M} \sum_{i=1}^{M} \frac{ \mathrm{d}\overline{\Pb}(\cdot \mid \kappa=i) }{ \mathrm{d}\overline{\Qb} } (\mathsf Y_1,\ldots,\mathsf Y_M) = \frac{1}{M} \sum_{i=1}^{M} \frac{ \mathrm{d} \Pb' }{ \mathrm{d} \Qb' } (\mathsf Y_i) \,.  \label{eq-explicit-form-LR-hidden-sample}
    \end{align}
    In addition, recall \eqref{eq-def-low-deg-Adv}. Define
    \begin{align*}
        \mathcal N_D=\big\{ f \in \mathcal P_D: f = 0 \mbox{ a.s. under } \Qb' \big\} \,.
    \end{align*}
    It is straightforward to verify that $\mathcal N_D$ is a linear subspace of $\mathcal P_D$. Also, consider the inner product $\langle f,g \rangle= \mathbb E_{\Qb'}[fg]$, denoting $\mathcal N_D^{\perp}$ to be the orthogonal space of $\mathcal N_D$ in $\mathcal P_D$, then $\mathcal N_D^{\perp}$ can be identified as a (finite-dimensional) Hilbert space where this inner product is non-degenerate (note that $\langle f,f \rangle>0$ for all $f \in \mathcal N_D^{\perp} \setminus \{0\}$). In addition, it is straightforward to check that $1 \in \mathcal N_D^{\perp}$ with $\langle 1,1 \rangle=1$. Thus, by applying the Schmidt orthogonalization procedure started with $1$, we can find a standard orthogonal basis $\{ f_{\alpha}: \alpha \in \Lambda \}$ of the space $\mathcal N_D^{\perp}$ and with $0\in\Lambda$ and $f_{0}=1$. Note that in some cases (e.g., when $\Qb'$ is a product measure), this standard orthogonal basis have a closed form; however, for general $\Qb'$ the explicit form of $\{ f_{\alpha} \}$ is often intractable. We first show that
    \begin{equation}{\label{eq-low-deg-Adv-trandsform}}
        \mathsf{Adv}_{\leq D}\big( \tfrac{ \mathrm{d}\Pb' }{ \mathrm{d}\Qb' } \big) = \Bigg( \sum_{\alpha\in\Lambda} \mathbb E_{\Pb'}[ f_{\alpha}(\mathsf Y) ]^2 \Bigg)^{\frac{1}{2}} \,.
    \end{equation}
    Indeed, as $\mathsf{Adv}_{\leq D}\big( \tfrac{\mathrm{d}\Pb'}{\mathrm{d}\Qb'} \big)=O(1)$, we see that for all $f \in \mathcal N_D$ it must holds that $\mathbb E_{\Pb'}[f]=0$ (since otherwise we will have $\frac{\mathbb E_{\Pb'}[f]}{\mathbb E_{\Qb'}[f^2]}=\infty$). Thus, we have
    \begin{align*}
        \mathsf{Adv}_{\leq D}\big( \tfrac{\mathrm{d}\Pb'}{\mathrm{d}\Qb'} \big) = \max_{ f \in \mathcal N_D^{\perp} } \frac{\mathbb E_{\Pb'}[f]}{\mathbb E_{\Qb'}[f^2]} \,.
    \end{align*}
    For any $f \in \mathcal N_D^{\perp}$, it can be uniquely expressed as
    \begin{align*}
        f=\sum_{ \alpha\in\Lambda } C_{\alpha} f_{\alpha} \,,
    \end{align*}
    where $C_{\alpha}$'s are real constants. Applying Cauchy-Schwartz inequality one gets
    \begin{align*}
        \frac{ \mathbb{E}_{\Pb'}[f] }{ \sqrt{\mathbb{E}_{\Qb'}[f^2]} } = \frac{ \sum_{ \alpha\in\Lambda } C_{\alpha} \mathbb{E}_{\mathbb{P}'}[f_{\alpha}(\mathsf Y)] }{ \sqrt{ \sum_{ \alpha\in\Lambda } C_{\alpha}^2} } \leq \Bigg( \sum_{\alpha\in\Lambda} \mathbb E_{\Pb'} [f_{\alpha}(\mathsf Y)]^2 \Bigg)^{1/2} \,,
    \end{align*}
    with equality holds if and only if $C_{\alpha} \propto \mathbb{E}_{\Pb'}[f_{\alpha}]$. This yields \eqref{eq-low-deg-Adv-trandsform}. Now, note that $\overline{\Qb}=(\Qb')^{\otimes M}$ is a product measure of $\Qb'$, there is a natural standard orthogonal basis under $\overline{\Qb}$, given by
    \begin{align*}
        \Bigg\{ \prod_{i=1}^{M} f_{\alpha_i}(\mathsf Y_i): \alpha_i \in \Lambda, \sum_{i=1}^{M} \operatorname{deg}(f_{\alpha_i}) \leq D \Bigg\} \,.
    \end{align*}
    Thus, similarly as in \eqref{eq-low-deg-Adv-trandsform}, we see that
    \begin{align}
        \Big( \mathsf{Adv}_{\leq D}\big( \tfrac{ \mathrm{d}\overline{\Pb} }{ \mathrm{d}\overline{\Qb} } \big) \Big)^2 = \sum_{ \substack{ (\alpha_1,\ldots,\alpha_M): \alpha_i \in \Lambda \\ \sum_{i=1}^{M} \operatorname{deg}(f_{\alpha_i}) \leq D } } \mathbb E_{ \overline{\Pb} }\Big[ \prod_{i=1}^{M} f_{\alpha_i}(\mathsf Y_i) \Big]^2 \,. \label{eq-low-deg-Adv-overline-relax-1}
    \end{align}
    In addition, using \eqref{eq-explicit-form-LR-hidden-sample}, we see from direct calculation that
    \begin{align}
        \mathbb E_{ \overline{\Pb} }\Big[ \prod_{i=1}^{M} f_{\alpha_i}(\mathsf Y_i) \Big] &= \mathbb E_{ \overline{\Qb} }\Big[ \prod_{i=1}^{M} f_{\alpha_i}(\mathsf Y_i) \cdot \frac{\mathrm{d}\overline{\Pb}}{\mathrm{d}\overline{\Qb}} (\mathsf Y_1,\ldots,\mathsf Y_M) \Big] \nonumber \\
        &\overset{\eqref{eq-explicit-form-LR-hidden-sample}}{=} \frac{1}{M} \sum_{i=1}^M \mathbb E_{ \overline{\Qb} }\Big[ \prod_{i=1}^{M} f_{\alpha_i}(\mathsf Y_i) \cdot \frac{\mathrm{d}\Pb'}{\mathrm{d}\Qb'} (\mathsf Y_i) \Big] \nonumber \\
        &= \begin{cases}
            1 \,, & (\alpha_1,\ldots,\alpha_M) = (0,\ldots,0) \,; \\
            \frac{1}{M} \mathbb E_{\Pb'}\big[ f_{\alpha_j}(\mathsf Y_j) \big] \,, & (\alpha_1,\ldots,\alpha_M) = (0,\ldots,0,\alpha_j,0,\ldots,0) \,; \\
            0 \,, & \mbox{otherwise} \,.
        \end{cases} \label{eq-cal-moment-hidden-sample}
    \end{align}
    Plugging \eqref{eq-cal-moment-hidden-sample} into \eqref{eq-low-deg-Adv-overline-relax-1}, we get that 
    \begin{align*}
        \eqref{eq-low-deg-Adv-overline-relax-1} &= 1 + \sum_{i=1}^{M} \sum_{ \alpha_i \in \Lambda\setminus\{ 0 \} } \Big( \frac{1}{M} \mathbb E_{\Pb'}\big[ f_{\alpha_i}(\mathsf Y_i) \big] \Big)^2 \\
        &\leq 1 + \frac{1}{M} \sum_{\alpha \in \Lambda}  \mathbb E_{\Pb'}\big[ f_{\alpha}(\mathsf Y_j) \big]^2 = 1 + \frac{1}{M} \cdot O(1) \overset{\eqref{eq-def-M-n}}{=} 1+o(1) \,,
    \end{align*}
    where in the second equality we use \eqref{eq-low-deg-Adv-trandsform} and the assumption that $\mathsf{Adv}_{\leq D}\big( \tfrac{ \mathrm{d}\Pb' }{ \mathrm{d}\Qb' } \big)=O(1)$. This completes our proof.
\end{proof}

We can now finish the proof of Theorem~\ref{main-thm-alg-contiguity}.
\begin{proof}[Proof of Theorem~\ref{main-thm-alg-contiguity}]
    Suppose on the contrary that there is an efficient algorithm $\mathcal A$ satisfying \eqref{eq-contradict-assumption}. Consider the hypothesis testing problem stated in Definition~\ref{def-hidden-sample}. Using \eqref{eq-behavior-Qb-hidden-sample} and \eqref{eq-behavior-Pb-hidden-sample}, we see that there is an efficient algorithm that achieves weak detection between $\overline{\Pb}$ and $\overline{\Qb}$, which contradicts with Lemma~\ref{lem-low-deg-hardness-hidden-sample} and Item~(2) in Conjecture~\ref{conj-low-deg}.
\end{proof}

\section{Partial recovery in correlated random graphs}{\label{sec:application-1}}

In this section, we will use the framework we established in Section~\ref{sec:revise-low-deg-conj} to show the hardness of partial matching in correlated random graphs, thus justifying Items~(1) and (2) in Theorem~\ref{MAIN-THM-informal}. To this end, we first state the precise meaning that an algorithm achieves partial matching.

\begin{defn}[Partial recovery algorithm in correlated \ER model]{\label{def-partial-recovery}}
    For two elements $\pi,\pi'\in\mathfrak S_n$, define
    \begin{equation}{\label{eq-def-overlap}}
        \mathsf{OV}(\pi,\pi')= \frac{1}{n} \sum_{i=1}^{n} \mathbf 1_{ \{ \pi(i)=\pi'(i) \} } \,.
    \end{equation}
    Given a sample $(A,B)$ from the law a pair of correlated random graphs in Definition~\ref{def-correlated-random-graph} (we denote this law of $\Pb_*$). We say an algorithm $\mathcal A$ achieves \emph{strong partial matching}, if it takes $(A,B)$ as input and outputs an estimator $\widehat{\pi}=\widehat{\pi}(A,B) \in \mathfrak S_n$ such that there exists a fixed constant $\iota>0$ with
    \begin{equation}{\label{eq-def-strong-partial-matching}}
        \Pb_*\big( \mathsf{OV}(\widehat\pi,\pi_*) \geq \iota \big) = 1-o(1) \,.
    \end{equation}
    We say an algorithm $\mathcal A$ achieves \emph{weak partial matching}, if it takes $(A,B)$ as input and outputs an estimator $\widehat{\pi}=\widehat{\pi}(A,B) \in \mathfrak S_n$ such that there exists a fixed constant $\iota>0$ with
    \begin{equation}{\label{eq-def-weak-partial-matching}}
        \Pb_*\big( \mathsf{OV}(\widehat\pi,\pi_*) \geq \iota \big) = \Omega(1) \,.
    \end{equation}
\end{defn}
\begin{remark}{\label{rmk-entrywise-est}}
    Given any estimator $\widehat{\pi}(A,B) \in \mathfrak S_n$, by defining $h_{i,j}=\mathbf 1_{\widehat{\pi}(i)=j}$ we get a family of estimators $\{ h_{i,j}: 1 \leq i, j \leq n\}$ such that
    \begin{enumerate}
        \item[(1)] $h_{i,j} \in \{ 0,1 \}$ for all $1 \leq i,j \leq n$ a.s. under $\Pb_*$;
        \item[(2)] $h_{i,1} + \ldots + h_{i,n}=1$ for all $1 \leq i \leq n$ a.s. under $\Pb_*$.
    \end{enumerate}
    In addition, suppose $\widehat{\pi}(A,B)$ achieves strong partial recovery, it is easy to check that
    \begin{enumerate}
        \item[(3)] $\Pb_*( h_{1,\pi_*(1)} + \ldots + h_{n,\pi_*(n)} \geq \iota n ) = 1-o(1)$.
    \end{enumerate}
    And suppose $\widehat{\pi}(A,B)$ achieves weak partial recovery, it is easy to check that
    \begin{enumerate}
        \item[(3')] $\Pb_*( h_{1,\pi_*(1)} + \ldots + h_{n,\pi_*(n)} \geq \iota n ) = \Omega(1)$.
    \end{enumerate}
    In the rest part of this section, we will use the estimator $\widehat{\pi}$ and the family of estimators $\{ h_{i,j}:1 \leq i,j \leq n \}$ interchangeably.
\end{remark}

Our result in this section can be stated as follows:
\begin{thm}{\label{main-thm-cor-ER}}
    Assuming Conjecture~\ref{conj-low-deg}, we have the following:
    \begin{enumerate}
        \item[(1)] For the correlated \ER graphs $\mathcal G(n,q,\rho)$ where $q=n^{-1+o(1)}$ and $\rho<\sqrt{\alpha}-\delta$ for a fixed constant $\delta>0$. There exists a constant $C$ such that no algorithm with running time $n^{D/(\log n)^C}$ that achieves weak partial matching, provided that 
        \begin{equation}{\label{eq-degree-assumption-cor-ER}}
            D = \exp\Big( o\big( \tfrac{\log n}{\log(nq)} \wedge \sqrt{\log n} \big) \Big) \,.
        \end{equation}
        \item[(2)] For the correlated SBMs $\mathcal S(n,\tfrac{\lambda}{n};k,\epsilon;s)$ where $\lambda=O(1)$ and $\epsilon^2 \lambda<1-\delta, s < \sqrt{\alpha}-\delta$ for a fixed constant $\delta>0$. There exists a constant $C$ such that no algorithm with running time $n^{D/(\log n)^C}$ that achieves strong partial matching, provided that 
        \begin{equation}{\label{eq-degree-assumption-cor-SBM}}
            D = n^{o(1)} \,.
        \end{equation}
    \end{enumerate}
\end{thm}

The main step of proving Theorem~\ref{main-thm-cor-ER} is to show the following proposition:
\begin{proposition}{\label{main-prop-cor-graphs}}
    Assuming Conjecture~\ref{conj-low-deg}, we have the following:
    \begin{enumerate}
        \item[(1)] If $(G,A,B) \sim \mathcal G(n,q,\rho)$ and let $\Pb_*$ to be the joint law of $(\pi_*,G,A,B)$ where $\pi_*$ is the latent matching. Suppose that $q,\rho,D$ satisfy the assumptions in Item~(1) of Theorem~\ref{main-thm-cor-ER}. Then there exists a constant $C$ such that for all $\{ h_{i,j}(A,B):1 \leq i,j \leq n \}$ that satisfies Items~(1),(2) in Remark~\ref{rmk-entrywise-est} and can be computed in time $n^{D/(\log n)^C}$, we have $\mathbb E_{\Pb_*}[ h_{i,\pi_*(i)} ]=o(1)$ for all $1 \leq i \leq n$.
        \item[(1)] If $(G,A,B) \sim \mathcal S(n,\tfrac{\lambda}{n};k,\epsilon,s)$ and let $\Pb_*$ to be the joint law of $(\pi_*,G,A,B)$ where $\pi_*$ is the latent matching. Suppose that $\lambda,k,\epsilon,s$ and $D$ satisfy the assumptions in Item~(2) of Theorem~\ref{main-thm-cor-ER}. Then there exists a constant $C$ and an event $\mathcal E$ such that $\Pb_*(\mathcal E)=\Omega(1)$, and for all for all $\{ h_{i,j}(A,B):1 \leq i,j \leq n \}$ that satisfies Items~(1),(2) in Remark~\ref{rmk-entrywise-est} and can be computed in time $n^{D/(\log n)^C}$, we have $\mathbb E_{\Pb_*}[ h_{i,\pi_*(i)} \mid \mathcal E ]=o(1)$ for all $1 \leq i \leq n$.
    \end{enumerate}
\end{proposition}

Clearly, based on Proposition~\ref{main-prop-cor-graphs}, we can deduce Theorem~\ref{main-thm-cor-ER} via a simple Markov inequality. The rest of this section is devoted to the proof of Proposition~\ref{main-prop-cor-graphs}. In the following subsections, our main focus is on proving Item~(1) of Proposition~\ref{main-prop-cor-graphs}. Given the similarity between the proofs of Item~(1) and Item~(2), for Item~(2) we will provide an outline with the main differences while adapting arguments from proving Item~(1) without presenting full details.

\subsection{Proof of Item~(1) in Proposition~\ref{main-prop-cor-graphs}}{\label{subsec:proof-item-1}}

This subsection is devoted to the proof of Item~(1) in Proposition~\ref{main-prop-cor-graphs}. Throughout this subsection, we will denote $\Pb_*$ to be the law of $(\pi_*,G,A,B)$ where $(G,A,B) \sim \mathcal G(n,q;\rho)$. We will also denote $\Pb$ to be the marginal law of $(A,B)$. In addition, we assume throughout this subsection that there exists a small constant $0<\delta<0.01$ such that
\begin{equation}{\label{eq-assumption-parameters-sec-3.1}}
    \rho^2<\alpha-\delta \,, \quad q=n^{-1+o(1)} \,, \quad \log D = o\Big( \tfrac{\log n}{\log(nq)} \wedge \sqrt{\log n} \Big) \,. 
\end{equation}
We first introduce some notations used in \cite{DDL23+}.

\begin{defn}\label{def-addmisible}
    Given a graph $H=H(V,E)$, define 
    \begin{equation}\label{eq-def-Phi}
        \Phi(H)={\big(n^{1+4/D} {D}^{20}\big)^{|V(H)|} \big(q D^6\big)^{|E(H)|}\,,}
    \end{equation}
    and the graph $H$ is said to be \emph{bad} if ${\Phi(H)<(\log n)^{-1}}$. Furthermore, we say a graph is \emph{admissible} if it contains no bad subgraph, and we say it is \emph{inadmissible} otherwise. 
    
    Denote $\mathcal E$ for the event that $G$ does not contain any bad subgraph with no more than $d^2$ vertices. In addition, let $\overline{\Pb}_*$ be the conditional version of $\Pb_*$ given $\mathcal E$, and let $\overline{\Pb}$ be the corresponding marginal distribution of $\overline{\Pb}_*$ on $(A,B)$.  
\end{defn}

We remark here that our definition of ``bad'' amounts to an atypically large edge density, with a carefully chosen quantitative threshold on ``large''. Roughly speaking, we expect that any subgraph with size no more than $D^2=n^{o(1)}$ of a sparse \ER graph has edge-to-vertex ratio $1+o(1)$. In the definition of $\Phi$, the term $n^{1+4/D} D^{20}$ should be thought as $n^{1+o(1)}$, and $qD^6$ as $n^{-1+o(1)}$. The $o(1)$ terms are tuned carefully so that for a typical subgraph $H$ of a sparse \ER graph, $\Phi(H)$ is much bigger than $1$. The choice of $(\log n)^{-1}$ as the $\Phi$-threshold for bad graph is somewhat arbitrary, which we will only need to be vanishing as $n\to\infty$. In \cite{DDL23+}, the authors showed that one the one hand, we have $\Pb_*(\mathcal E)=1-o(1)$ and thus $\operatorname{TV}(\Pb_*,\overline{\Pb}_*), \operatorname{TV}(\Pb,\overline{\Pb})=o(1)$; on the other hand, we have $\mathsf{Adv}_{\leq D}\big( \tfrac{\mathrm{d}\overline{\Pb}}{\mathrm{d}\Qb} \big)=O_{\delta}(1)$, thus verifying the low-degree hardness for the detection problem. The first step of our proof is to generalize the result in \cite{DDL23+}, as incorporated in the following lemma.

\begin{lemma}{\label{lem-bound-low-deg-Adv-conditional}}
    For all $1\leq i,j \leq n$, we have $\mathsf{Adv}_{\leq D}\big( \tfrac{ \mathrm{d}\overline{\Pb}(\cdot \mid \pi_*(i)=j) }{ \mathrm{d}\Qb } \big)=O_{\delta}(1)$.
\end{lemma}

The proof of Lemma~\ref{lem-bound-low-deg-Adv-conditional} is quite technical and thus we postpone it to Section~\ref{subsec:proof-lem-3.5} of the appendix. Now, based on Lemma~\ref{lem-bound-low-deg-Adv-conditional}, we establish the following result, which basically suggests that it is impossible to obtain a ``good approximation'' of $\mathbf 1_{\{\pi_*(i)=j\}}$ in some sense.

\begin{lemma}{\label{lem-low-deg-hardness-partial-matching-new}}
    Assuming Conjecture~\ref{conj-low-deg}, there exists a constant $C$ such that for all $1 \leq i \leq n$ and all statistics $\{ g_{i,j}=g_{i,j}(A,B): 1\leq j \leq n \}$ such that each $g_{i,j}(A,B)$ can be computed in time $n^{D/(\log n)^C}$ and
    \begin{align}
        \sum_{ 1 \leq j \leq n } g_{i,j}(A,B) = o(n) \mbox{ for all } 1 \leq i \leq n \mbox{ and for all graphs } A,B \mbox{ on } [n] \,, \label{eq-regularize-condition}
    \end{align}
    it holds that
    \begin{equation}{\label{eq-low-deg-hardness-partial-matching-new}}
        \sum_{j=1}^{n} \mathbb E_{\Pb_*}\Big[ \big( \mathbf 1_{\{\pi_*(i)=j\}}-g_{i,j} \big)^2 \Big] \geq 1-o(1) \,.
    \end{equation}
\end{lemma}
\begin{proof}
    Without loss of generality, we may assume that $i=1$ in the following proof. Suppose on the contrary that there are statistics $\{ g_j=g_j(A,B): 1\leq j \leq n \}$ that can be computed in time $n^{D/(\log n)^C}$ such that $\sum_{ 1 \leq j \leq n } g_{j}(A,B) = o(n)$ for all graphs $A,B$ and 
    \begin{equation}{\label{eq-low-deg-hardness-partial-matching-converse}}
        \sum_{j=1}^{n} \mathbb E_{\Pb_*}\Big[ \big( \mathbf 1_{\{\pi_*(1)=j\}}-g_j \big)^2 \Big] \leq 1-c \mbox{ for some constant } c>0 \,.
    \end{equation}
    Without loss of generality, we may assume that $0 \leq g_i \leq 1$, since otherwise we may replace $g_i$ with $\min\{ \max\{ g_i,0\}, 1\}$, which will only make the left hand side of \eqref{eq-low-deg-hardness-partial-matching-converse} smaller. Denote 
    \begin{equation}{\label{eq-def-Lambda}}
        \Lambda := \Big\{ 1 \leq j \leq n: \mathbb E_{\Pb_*}\big[ (\mathbf 1_{\{\pi_*(1)=j\}}-g_j)^2 \big] \leq \frac{1-\frac{c}{2}}{n} \Big\} \,.
    \end{equation}
    Using Markov inequality, we see that
    \begin{equation}{\label{eq-bound-card-Lambda}}
        |\Lambda| \geq n - \frac{ 1-c }{ 1-\frac{c}{2} } n \geq \frac{cn}{2} \,.
    \end{equation}
    Recall that $\overline{\Pb}=\Pb(\cdot \mid \mathcal E)$ where $\mathcal E$ is independent to the latent matching $\pi_*$ and $\Pb_*(\mathcal E)=1-o(1)$. Thus, it is straightforward to check that for all $i \in \Lambda$
    \begin{align*}
        \mathbb E_{\overline{\Pb}}\big[ (1-g_i)^2 \mid \pi_*(1)=i \big] \leq [1+o(1)] n \cdot \mathbb E_{\Pb_*}\big[ (\mathbf 1_{\{\pi_*(1)=i\}}-g_i)^2 \big] \leq 1-\tfrac{c}{2} \,.
    \end{align*}
    Thus, using Markov inequality we see that (note that $0 \leq g_i \leq 1$) 
    \begin{align}\label{eq-behavior-Pb'}
        \overline{\Pb}( g_i > \tfrac{c}{2} \mid \pi_*(1)=i ) \geq \Omega(1) \,.
    \end{align}
    However, using Lemma~\ref{lem-bound-low-deg-Adv-conditional} and Theorem~\ref{main-thm-alg-contiguity}, we see that (assuming Conjecture~\ref{conj-low-deg}) we have 
    \begin{align*}
        \overline{\Pb}(\cdot \mid \pi_*(1)=i) \lhd_{\leq D} \Qb \,.
    \end{align*}
    Thus we must have 
    \begin{align}
        \Qb( g_i > \tfrac{c}{2} ) \geq \Omega(1) \,. \label{eq-behavior-Qb'}
    \end{align}
    To this end, define $g=g_1+\ldots+g_n$. Using $0 \leq g \leq n$, we see that \eqref{eq-behavior-Qb'} yields that $\mathbb E_\Qb[g]=\Omega(n)$, which violates our assumption \eqref{eq-regularize-condition}. This implies that \eqref{eq-low-deg-hardness-partial-matching-converse} is impossible and thus completes our proof.
\end{proof}

Now we can finish the proof of Item~(1) in Proposition~\ref{main-prop-cor-graphs} assuming Conjecture~\ref{conj-low-deg}.

\begin{proof}[Proof of Item~(1) in Proposition~\ref{main-prop-cor-graphs} assuming Conjecture~\ref{conj-low-deg}]
    Suppose on the contrary there are statistics $\{ f_j: 1\leq j \leq n \}$ that can be computed in time $n^{D/(\log n)^C}$ satisfying Items~(1) and (2) in Definition~\ref{def-partial-recovery} with $\mathbb E_{\Pb_*}[f_{\pi_*(i)}] \geq 1-c$ for some fixed constant $0<c<0.01$. We first claim that we may assume that
    \begin{align}
        f_{i,j}(A,B) \in \{ 0,1 \} \mbox{ and } \sum_{1 \leq j \leq n} f_{i,j}(A,B) \in \{ 0,1 \} \mbox{ for all graphs } A,B \mbox{ on } [n] \,.  \label{eq-stronger-regularize-condition}
    \end{align}
    In fact, denote 
    \begin{align*}
        \mathcal A = \Big( \cap_{1 \leq i,j \leq n} \Big\{ f_{i,j}(A,B) \in \{ 0,1 \} \Big\} \Big) \bigcap \Big( \cap_{1 \leq i \leq n} \Big\{ \sum_{1 \leq j \leq n} f_{i,j}(A,B) = 1 \Big\} \Big) \,,
    \end{align*}
    we may consider the truncated statistics
    \begin{align*}
        f_{i,j}'(A,B) = f_{i,j}(A,B) \cdot \mathbf 1_{\mathcal A} \,.
    \end{align*}
    It is straightforward to check that $\{ f_{i,j}'(A,B): 1 \leq i,j \leq n \}$ satisfies \eqref{eq-stronger-regularize-condition} and we still have
    \begin{align*}
        \mathbb E_{\Pb_*}\big[ f_{i,j}'(A,B) \big] \geq 1-c \,.
    \end{align*}
    Thus, we may assume that \eqref{eq-stronger-regularize-condition} holds without loss of generality. To this end, using \eqref{eq-stronger-regularize-condition}, we see that $f_{1,j} f_{1,k}=0$ for all $k \neq j$, and thus
    \begin{align}
        1 = \mathbb E_{\Pb}\big[ (f_{1,1}+\ldots+f_{1,n})^2 \big] = \sum_{j=1}^n \mathbb E_{\Pb}\big[ f_{1,j}^2 \big] \,. \label{eq-3.3-new}
    \end{align}
    In addition, we have
    \begin{align}
        \mathbb E_{\Pb_*}\big[ f_{1,\pi_*(i)} \big] = \frac{1}{n} \sum_{j=1}^{n} \mathbb E_{\Pb_*}\big[ f_{1,j} \mid \pi_*(i)=j \big] \geq c \,. \label{eq-3.4-new}
    \end{align}
    Thus, for all $\lambda\in [0,1]$ we have
    \begin{align*}
        \sum_{j=1}^{n} \mathbb E_{\Pb_*}\Big[ \big( \mathbf 1_{\{\pi_*(i)=j\}} - \tfrac{1-\lambda}{n} - \lambda f_{1,j} \big)^2 \Big] \overset{\eqref{eq-3.3-new},\eqref{eq-3.4-new}}{\leq}\ 1+ \lambda^2-2c\lambda + O\big( \tfrac{1}{n} \big) \,.
    \end{align*}
    Thus, by choosing $\lambda=\lambda(c)$ to be a sufficiently small positive constant we get that 
    \begin{align*}
        \sum_{j=1}^{n} \mathbb E_{\Pb_*}\Big[ \big( \mathbf 1_{\{\pi_*(i)=j\}}-g_j \big)^2 \Big] = 1-\Omega(1) \mbox{ where } g_j=\tfrac{1-\lambda}{n}+\lambda f_{1,j} \,,
    \end{align*}
    contradicting to Lemma~\ref{lem-low-deg-hardness-partial-matching-new}. This leads to the desired result.
\end{proof}

\subsection{Proof of Item~(2) in Proposition~\ref{main-prop-cor-graphs}}{\label{subsec:proof-item-2}}

This subsection is devoted to the proof of Item~(2) in Proposition~\ref{main-prop-cor-graphs}. Recall Definitions~\ref{def-SBM} and \ref{def-correlated-SBM}. Throughout this subsection, we will denote $\Pb_*$ to be the joint law of $(\pi_*,\sigma_*,G,A,B)$ where $(G,A,B) \sim \mathcal S(n,\tfrac{\lambda}{n};k,\epsilon;s)$ and $\Pb$ the marginal law of $(A,B)$. In addition, we denote $\Qb$ to be the law of a pair of independent \ER models $\mathcal G(n,\tfrac{\lambda s}{n})$. In addition, we assume throughout this subsection that there exists a small constant $0<\delta<0.01$ such that
\begin{equation}{\label{eq-assumption-parameter-SBM}}
    s<\sqrt{\alpha}-\delta \,, \quad \epsilon^2 \lambda s < 1-\delta \,.
\end{equation}
We also choose a sufficiently large constant $N=N(k,\lambda,\delta,\epsilon,s) \geq 2 / \delta$ such that 
\begin{equation}{\label{eq-def-N}}
    \begin{aligned}
        & (\sqrt{\alpha}-\delta) (1+\epsilon^{N}k) \leq \sqrt{\alpha} - \delta/2 \,; \quad 10k(1-\delta)^N \leq (1-\delta/2)^{N} \,; \\
        & (\sqrt{\alpha}-\delta/2)(1+(1-\delta/2)^N)^2 \leq \sqrt{\alpha}-\delta/4 \,; \quad (1-\delta / 2)^N (N+1) \leq 1\,. 
    \end{aligned}
\end{equation}

We first show how to construct the event $\mathcal E$ in Item~(2) in Proposition~\ref{main-prop-cor-graphs}.

\begin{defn}\label{def-addmisible-SBM}
    Denote $\Tilde{\lambda}=\lambda\vee 1$. Given a graph $H=H(V,E)$, define 
    \begin{equation}\label{eq-def-Phi-SBM}
        \Upsilon(H) = \Big( \frac{2 \Tilde{\lambda}^2 k^2 n}{D^{50}} \Big)^{|V(H)|} \Big( \frac{ 1000 \Tilde{\lambda}^{20} k^{20} D^{50} }{ n } \Big)^{|E(H)|}  \,.
    \end{equation}
    Then we say the graph $H$ is \emph{bad} if $\Upsilon(H) < (\log n)^{-1}$, and we say a graph $H$ is \emph{self-bad} if $H$ is bad and $\Upsilon(H)<\Upsilon(K)$ for all $K \subset H$. Furthermore, we say that a graph $H$ is \emph{admissible} if it contains no bad subgraph and $\mathtt C_j(H) =\emptyset$ for $j \leq N$; we say $H$ is \emph{inadmissible} otherwise. Denote $\mathcal E = \mathcal E^{(1)} \cap \mathcal E^{(2)}$, where 
    $\mathcal E^{(1)}$ is the event that $G$ does not contain any bad subgraph with no more than $D^3$ vertices, and $\mathcal E^{(2)}$ is the event that $G$ does not contain any cycles with length at most $N$. 
\end{defn}

\begin{defn}\label{def-G'-P'}
    List all self-bad subgraphs of $\mathcal K_n$ with at most $D^3$ vertices and all cycles of $\mathcal K_n$ with lengths at most $N$ in an arbitrary but prefixed order $(B_1,\ldots,B_{\mathtt M})$. Define a stochastic block model with ``bad graphs" removed as follows: (1) sample $G \sim \mathcal S(n,\tfrac{\lambda}{n};k,\epsilon)$; (2) for each $\mathtt 1 \leq \mathtt i \leq \mathtt M$ such that $B_{\mathtt i} \subset G$, we independently uniformly remove one edge in $B_{\mathtt i}$. The unremoved edges in $G$ constitute a graph $G'$, which is the output of our modified stochastic block model. Clearly, from this definition $G'$ does not contain any cycle of length at most $N$ nor any bad subgraph with at most $D^3$ vertices. Conditioned on $G'$ and $\pi_*$, we define 
    \[
    A'_{i,j} = G'_{i,j}J'_{i,j}, B'_{i,j} = G'_{\pi_*^{-1}(i),\pi_*^{-1}(j)} K'_{i,j} \,,
    \]
    where $J'$ and $K'$ are independent Bernoulli variables with parameter $s$. Let $\widetilde{\mathbb P}_* = \widetilde{\mathbb P}_{*,n}$ be the law of $(\sigma_*,\pi_*,G,G',A',B')$ and denote $\widetilde{\Pb}=\widetilde{\Pb}_n$ the marginal law of $(A',B')$.
\end{defn}

It was shown in \cite[Lemmas~4.2 and 4.4]{CDGL24+} that 
\begin{align*}
    \Pb_*(\mathcal E) = \Omega(1) \mbox{ and } \operatorname{TV}(\widetilde{\Pb},\Pb(\cdot \mid \mathcal E))=o(1) \,.
\end{align*}
Similarly as in Section~\ref{subsec:proof-item-1}, our first step is to show the following lemma.

\begin{lemma}{\label{lem-bound-low-deg-Adv-conditional-SBM}}
    We have $\mathsf{Adv}_{\leq D}\big( \tfrac{ \mathrm{d}\widetilde{\Pb}(\cdot \mid \pi_*(i)=j) }{ \mathrm{d}\Qb } \big)=O_{\delta,k}(1)$.
\end{lemma}

The proof of Lemma~\ref{lem-bound-low-deg-Adv-conditional-SBM} is incorporated in Section~\ref{subsec:proof-lem-3.9} of the appendix.
Based on Lemma~\ref{lem-bound-low-deg-Adv-conditional-SBM}, we can deduce our main result just as how we deduce Theorem~\ref{main-thm-cor-ER} from Lemma~\ref{lem-bound-low-deg-Adv-conditional}. The only difference is that we will replace all $\overline{\Pb}$ with $\widetilde{\Pb}$ and replace all $\Pb$ with $\Pb(\cdot \mid \mathcal E)$ so we omit further details here.

\section{Detection in correlated SBMs}{\label{sec:application-2}}

In this section, we will use the framework we established in Section~\ref{sec:revise-low-deg-conj} in stochastic block models below KS-threshold. The main results of this section is incorporated as follows.

\begin{thm}{\label{main-thm-SBM}}
    For any constant $K \in \mathbb N$, denote $\Pb$ to be the law of $K$ independent stochastic block models $\mathcal S(n,\tfrac{\lambda}{n};k,\epsilon)$ and denote $\Qb$ to be the law of $K$ independent \ER graphs $\mathcal G(n,\tfrac{\lambda}{n})$. Then, assuming Conjecture~\ref{conj-low-deg}, for any $\delta>0$ there exists $\lambda_0=\lambda_0(\delta,k)$ to be a sufficiently large constant such that when $\epsilon^2 \lambda < 1-\delta$ and $\lambda>\lambda_0$, we have we have $\Pb \bowtie_{\leq D} \Qb$ for any $D=n^{o(1)}$.
\end{thm}

Our result has an immediate corollary in the detection problem between a pair of correlated SBMs and a pair of independent SBMs, as incorporated in the following corollary.

\begin{cor}{\label{cor-correlated-SBM}}
    Assuming Conjecture~\ref{conj-low-deg}, when $\epsilon^2 \lambda s<1-\delta$, $s<\sqrt{\alpha}-\delta$ and $\lambda>\lambda_0(\delta,k)$ there is no algorithms with polynomial running time that can strongly distinguish $\mathcal S(n,\tfrac{\lambda}{n};k,\epsilon;s)$ and two independent $\mathcal S(n,\tfrac{\lambda s}{n};k,\epsilon)$. 
\end{cor}
\begin{proof}
    Denote $\mu_{\mathsf{CorSBM}}, \mu_{\mathsf{IndSBM}}$ and $\mu_{\mathsf{IndER}}$ to be the law of a pair of correlated SBMs $\mathcal S(n,\tfrac{\lambda}{n};k,\epsilon;s)$, a pair of independent SBMs $\mathcal S(n,\tfrac{\lambda s}{n};k,\epsilon)$ and a pair of independent \ER graphs $\mathcal G(n,\tfrac{\lambda s}{n})$, respectively. Suppose on the contrary that there exists an efficient algorithm $\mathcal A$ such that 
    \begin{equation}{\label{eq-4.1}}
        \mu_{\mathsf{CorSBM}}\big( \mathcal A(A,B)=1 \big) = 1-o(1) \,, \quad \mu_{\mathsf{IndSBM}}\big( \mathcal A(A,B)=0 \big) = 1-o(1) \,.
    \end{equation}
    Using Theorem~\ref{main-thm-SBM} with $K=2$, we see that $\mu_{\mathsf{IndSBM}} \bowtie_{\leq D} \mu_{\mathsf{IndER}}$ for $D=n^{o(1)}$. Thus
    \begin{align*}
        \mu_{\mathsf{IndER}} \big( \mathcal A(A,B)=0 \big) = 1-o(1) \,.
    \end{align*}
    Thus, this algorithm $\mathcal A$ strongly distinguish $\mu_{\mathsf{CorSBM}}$ and $\mu_{\mathsf{IndER}}$. This contradicts Conjecture~\ref{conj-low-deg} and the low-degree hardness established in \cite[Theorem~1.3]{CDGL24+}.
\end{proof}

The rest part of this section is devoted to the proof of Theorem~\ref{main-thm-SBM}. For notational simplicity, in the following we will only prove the case where $K=1$ and the proof for general $K$ is similar. Note that $\epsilon^2 \lambda<1-\delta$ and $\lambda>\lambda_0$ implies that $\epsilon < \epsilon_0= \lambda_0^{-1/2}$. We choose $\lambda_0=\lambda_0(\delta,k)$ to be a sufficient large constant such that
\begin{equation}{\label{eq-choice-lambda-0}}
    (1-\delta)^{-1} \cdot \tfrac{ (k-1)\sqrt{1-\epsilon} + \sqrt{1+\epsilon(k-1)} }{ k } \geq (1-\delta/2)^{-1} \mbox{ for all } \epsilon< \epsilon_0= \lambda_0^{-1/2} \,. 
\end{equation}
In the rest part of this section we will always assume that
\begin{equation}{\label{eq-assumption-parameter}}
    D=n^{o(1)} \,, \lambda> \lambda_0 \mbox{ and } \epsilon^2 \lambda < 1-\delta \mbox{ for some constant } 0<\delta<0.01 \,.
\end{equation}
Clearly, using Theorem~\ref{main-thm-alg-contiguity}, it suffices to show that under \eqref{eq-assumption-parameter} and $\lambda_0$ we have
\begin{align*}
    \mathsf{Adv}_{\leq D}\big( \tfrac{ \mathrm{d}\Pb }{ \mathrm{d}\Qb } \big) = O_{\delta,k}(1) \mbox{ and } \mathsf{Adv}_{\leq D}\big( \tfrac{ \mathrm{d}\Qb }{ \mathrm{d}\Pb } \big) = O_{\delta,k}(1)  \,.
\end{align*}
Indeed, it has been shown in \cite{HS17} that $\mathsf{Adv}_{\leq D}\big( \tfrac{ \mathrm{d}\Pb }{ \mathrm{d}\Qb } \big) = O_{\delta,k}(1)$ provided with \eqref{eq-assumption-parameter}. It remains to show that under \eqref{eq-assumption-parameter} we have (note that now $\Pb$ is the planted measure)
\begin{equation}{\label{eq-final-goal-SBM}}
    \mathsf{Adv}_{\leq D}\big( \tfrac{ \mathrm{d}\Qb }{ \mathrm{d}\Pb } \big) = \sup_{f \in \mathcal P_{D}} \frac{\mathbb E_{\Qb}[f]}{ \sqrt{\mathbb E_{\Pb}[f^2]} } = O_{\delta,k}(1) \,.
\end{equation}
We point out that our approach to proving \eqref{eq-final-goal-SBM} is based on the the work \cite{SW25}. To this end, define
\begin{equation}{\label{eq-def-omega}}
    \omega(\sigma_i,\sigma_j) = 
    \begin{cases}
        k-1 \,, & \sigma_i = \sigma_j \,; \\
        -1 \,, & \sigma_i \neq \sigma_j 
    \end{cases}
\end{equation}
In addition, for all $S \Subset \mathcal{K}_n$ define
\begin{equation}{\label{eq-def-phi-S}}
    \phi_{S}\big(\{G_{i,j}\}\big) = \prod_{(i,j)\in E(S)} \frac{G_{i,j}-\frac{\lambda}{n}}{\sqrt{\frac{\lambda}{n}(1-\frac{\lambda}{n})}} \,.
\end{equation}
It is well known in \cite{HS17} that $\{ \phi_S: S \Subset \mathcal{K}_n, |E(S)|\leq D \}$ constitutes a standard orthogonal basis of $\mathcal P_{D}$ under $\Qb$. Thus, each $f \in \mathcal P_{D}$ can be written as 
\begin{equation}{\label{eq-expansion-f}}
    f(G) = \sum_{ S \Subset \mathcal{K}_n, |E(S)|\leq D } \widehat{f}_{S} \cdot \phi_S(G) \,,
\end{equation}
which means that $f$ is uniquely characterized by a vector $\widehat{f}$ indexed by $\{ S \Subset \mathcal{K}_n: |E(S)|\leq D \}$. In addition, direct calculation yields that
\begin{equation}{\label{eq-exp-Qb-phi}}
    \mathbb E_{\Qb}\big[ \phi_{S}(G) \big] = \mathbf 1_{\{ S=\emptyset \}} \,.
\end{equation}
Thus, we have
\begin{equation}{\label{eq-exp-Qb-f}}
    \mathbb E_{\Qb}\big[ f \big] \overset{\eqref{eq-expansion-f}}{=} \widehat{f}_{\emptyset} = \langle \widehat{f}, c \rangle \,,
\end{equation}
where $c$ is a vector indexed by $\{ S \Subset \mathcal{K}_n: |E(S)|\leq D \}$ with 
\begin{equation}{\label{eq-def-c}}
    c_{S} = \mathbf 1_{ \{ S = \emptyset \} } \,.
\end{equation}
We now turn to $\mathbb E_{\Pb}[f^2]$. For any $\sigma\in [k]^n$ and $S \Subset \mathcal{K}_n$, define
\begin{equation}{\label{eq-def-psi-sigma,S}}
    \psi_{\sigma,S}\big(\{G_{i,j}\}\big) = k^{\frac{n}{2}} \mathbf 1_{ \sigma_*=\sigma } \cdot \prod_{(i,j)\in E(S)} \frac{G_{i,j}-\frac{(1+\epsilon\omega(\sigma_i,\sigma_j)) \lambda}{n}}{ \sqrt{ \frac{(1+\epsilon\omega(\sigma_i,\sigma_j))\lambda}{n} (1-\frac{(1+\epsilon\omega(\sigma_i,\sigma_j))\lambda}{n}) } }
\end{equation}
We can check that $\{ \psi_{\sigma,S}: \sigma\in [k]^n, S \Subset \mathcal{K}_n \}$ is standard orthogonal under $\Pb_*$, i.e., we have
\begin{equation}{\label{eq-standard-orthogonal}}
    \mathbb E_{\Pb_*}\big[ \psi_{\sigma,S} \psi_{\sigma',S'} \big] = \mathbf 1_{ \{ \sigma=\sigma',S=S' \} } \,.
\end{equation}

\begin{lemma}{\label{lem-cal-cross-phi-psi}}
    We have
    \begin{equation}{\label{eq-cross-phi-psi}}
        \mathbb E_{\Pb}\big[ \phi_S(G)\psi_{\sigma,H}(G) \big] = \frac{ \mathbf 1_{H \subset S} }{ k^{\frac{n}{2}} } \prod_{(i,j) \in E(H)} \mathtt h(\sigma_i,\sigma_j) \prod_{ (i,j) \in E(S) \setminus E(H) } \omega(\sigma_i,\sigma_j) \sqrt{\tfrac{\epsilon^2 \lambda}{n}}  \,,
    \end{equation}
    where
    \begin{equation}{\label{eq-def-mathtt-h}}
        \mathtt h(\sigma_i,\sigma_j) = \sqrt{ \frac{ (1-\frac{(1+\epsilon\omega(\sigma_i,\sigma_j))\lambda}{n}) (1+\epsilon\omega(\sigma_i,\sigma_j)) }{ 1-\frac{\lambda}{n} } } \,.
    \end{equation}
\end{lemma}
\begin{proof}
    Note that using \eqref{eq-def-phi-S} and \eqref{eq-def-psi-sigma,S}, we have that $\mathbb E_{\Pb}[\phi_S(G)\psi_{\sigma,H}(G)]$ equals
    \begin{align}
        & \mathbb E_{\Pb}\Bigg[ k^{\frac{n}{2}} \mathbf 1_{ \sigma_*=\sigma } \cdot \prod_{(i,j)\in E(S)} \frac{G_{i,j}-\frac{\lambda}{n}}{\sqrt{\frac{\lambda}{n}(1-\frac{\lambda}{n})}} \cdot \prod_{(i,j)\in E(H)} \frac{G_{i,j}-\frac{(1+\epsilon\omega(\sigma_i,\sigma_j)) \lambda}{n}}{ \sqrt{ \frac{(1+\epsilon\omega(\sigma_i,\sigma_j))\lambda}{n} (1-\frac{(1+\epsilon\omega(\sigma_i,\sigma_j))\lambda}{n}) } } \Bigg] \nonumber \\
        =\ & k^{-\frac{n}{2}} \mathbb E_{\Pb_{\sigma}}\Bigg[ \prod_{(i,j)\in E(S)} \frac{G_{i,j}-\frac{\lambda}{n}}{\sqrt{\frac{\lambda}{n}(1-\frac{\lambda}{n})}} \cdot \prod_{(i,j)\in E(H)} \frac{G_{i,j}-\frac{(1+\epsilon\omega(\sigma_i,\sigma_j)) \lambda}{n}}{ \sqrt{ \frac{(1+\epsilon\omega(\sigma_i,\sigma_j))\lambda}{n} (1-\frac{(1+\epsilon\omega(\sigma_i,\sigma_j))\lambda}{n}) } } \Bigg] \,. \label{eq-cross-relax-1}
    \end{align}
    Clearly, when $H \not\subset S$ we have \eqref{eq-cross-relax-1} cancels to $0$, since (denote $(i,j)$ to be an element in $E(H) \setminus E(S)$) under $\Pb_{\sigma}$ we have $\mathbb E_{\Pb_{\sigma}}[ G_{i,j}-\tfrac{(1+\epsilon\omega(\sigma_i,\sigma_j))\lambda}{n} ]=0$ and all the edges $\{ G_{i,j}: (i,j) \in \operatorname{U}_n \}$ are independent. In addition, from direct calculation we have that
    \begin{align*}
        & \mathbb E_{\Pb_\sigma}\big[ G_{i,j}-\tfrac{\lambda}{n} \big] = \tfrac{ \epsilon\lambda \omega(\sigma_i,\sigma_j) }{ n } \,, \\
        & \mathbb E_{\Pb_\sigma}\big[ (G_{i,j}-\tfrac{\lambda}{n})( G_{i,j}-\tfrac{(1+\epsilon\omega(\sigma_i,\sigma_j))\lambda}{n} ) \big] = \tfrac{(1+\epsilon\omega(\sigma_i,\sigma_j))\lambda}{n} \big( 1-\tfrac{(1+\epsilon\omega(\sigma_i,\sigma_j))\lambda}{n} \big) \,.
    \end{align*}
    Plugging this result into \eqref{eq-cross-relax-1}, we get \eqref{eq-cross-phi-psi} and thus completes the proof.
\end{proof}

Based on Lemma~\ref{lem-cal-cross-phi-psi}, we define a matrix $M$ with rows indexed by $\{ S: S \Subset \mathcal{K}_n, |E(S)|\leq D \}$ and columns indexed by $\{ (\sigma,S): \sigma\in [k]^n, S \Subset \mathcal{K}_n, |E(S)| \leq D \}$, and entries given by
\begin{equation}{\label{eq-def-matrix-M}}
    M_{ S;(\sigma,H) } = \frac{ \mathbf 1_{H \subset S} }{ k^{\frac{n}{2}} } \prod_{(i,j) \in E(H)} \mathtt h(\sigma_i,\sigma_j) \prod_{ (i,j) \in E(S) \setminus E(H) } \omega(\sigma_i,\sigma_j) \sqrt{\tfrac{\epsilon^2 \lambda}{n}} \,.
\end{equation}
From Parserval's inequality, we see that
\begin{align}
    \mathbb E_{\Pb}[f^2] &\geq \sum_{ \substack{ \sigma\in [k]^n, H \Subset \mathcal{K}_n \\ |E(H)| \leq D } } \mathbb E_{\Pb}[ f \cdot \psi_{\sigma,H} ]^2 \overset{\eqref{eq-expansion-f}}{=} \sum_{ \substack{ \sigma\in [k]^n, H \Subset \mathcal{K}_n \\ |E(H)| \leq D } } \Big( \sum_{ \substack{ S \Subset \mathcal{K}_n \\ |E(S)| \leq D } } \widehat{f}_S \mathbb E_{\Pb}[ \phi_{S} \cdot \psi_{\sigma,H} ] \Big)^2 \nonumber \\
    &= \sum_{ \substack{ \sigma\in [k]^n, H \Subset \mathcal{K}_n \\ |E(H)| \leq D } } \Big( \sum_{ \substack{ S \Subset \mathcal{K}_n \\ |E(S)| \leq D } } \widehat{f}_S M_{S;(\sigma,H)} \Big)^2 = \big\| \widehat{f} M \big\|^2 \,. \label{eq-var-Pb-relax-1}
\end{align}
Thus, we have
\begin{equation}{\label{eq-relax-to-linear-equation}}
    \mathsf{Adv}_{\leq D}\big( \tfrac{\mathrm{d}\Qb}{\mathrm{d}\Pb} \big) = \sup_{f \in \mathcal P_{D}} \Bigg\{ \frac{\mathbb E_{\Qb}[f]}{ \sqrt{\mathbb E_{\Pb}[f^2]} } \Bigg\} \leq \sup_{ \widehat{f} } \Bigg\{ \frac{ \langle \widehat f,c \rangle }{ \| \widehat{f} M \| } \Bigg\} \leq \inf_{ M u^{\top} =c } \big\{ \| u \| \big\} \,,
\end{equation}
where the last inequality follows from the fact that for $M u_0^{\top}=c$ we have
\begin{align*}
    \langle \widehat f,c \rangle = \langle \widehat f,M u_0^{\top} \rangle = \langle \widehat f M, u_0 \rangle \leq \| u_0 \| \cdot \| \widehat f M \| \,.
\end{align*}
Regarding \eqref{eq-relax-to-linear-equation}, it suffices to show that there exists $M u^{\top}=c$ and $\|u\|=O_{\delta,k}(1)$. Note that $M u^{\top}=c$ is equivalent to 
\begin{equation}{\label{eq-linear-equation}}
    \sum_{\sigma\in [k]^n} \sum_{H \subset S} \frac{ u_{\sigma,H} }{ k^{\frac{n}{2}} } \prod_{(i,j) \in E(H)} \mathtt h(\sigma_i,\sigma_j) \prod_{ (i,j) \in E(S) \setminus E(H) } \omega(\sigma_i,\sigma_j) \sqrt{\tfrac{\epsilon^2 \lambda}{n}} = \mathbf 1_{ S=\emptyset } \,.
\end{equation}
The first step of our analysis is to simplify \eqref{eq-linear-equation}. Based on Lemma~\ref{lem-leaf-cancellation}, we see that for all $H \subset S$ with $\mathcal L(S) \setminus V(H) \neq \emptyset$, we have (denote $\mathtt V=V(H)$)
\begin{align*}
    &\sum_{\sigma\in [k]^n} u_{\sigma,H} \prod_{(i,j) \in E(H)} \mathtt h(\sigma_i,\sigma_j) \prod_{ (i,j) \in E(S) \setminus E(H) } \omega(\sigma_i,\sigma_j) \sqrt{\tfrac{\epsilon^2 \lambda}{n}} \\
    =\ & \big( \tfrac{\epsilon^2 \lambda}{n} \big)^{ \frac{|E(S)|-|E(H)|}{2} } \sum_{\sigma_{\mathtt V}\in [n]^{\mathtt V}} \prod_{(i,j) \in E(H)} \mathtt h(\sigma_i,\sigma_j) \times \sum_{ \sigma_{ [n] \setminus \mathtt V } \in [k]^{ [n] \setminus \mathtt V } } \prod_{ (i,j) \in E(S) \setminus E(H) } \omega(\sigma_i,\sigma_j) = 0 \,. 
\end{align*}
Thus, \eqref{eq-linear-equation} can be further simplified to
\begin{equation}{\label{eq-linear-equation-simplified}}
    \sum_{\sigma\in [k]^n} \sum_{ \substack{ H \subset S \\ \mathcal L(S) \subset V(H) } } \frac{ u_{\sigma,H} }{ k^{\frac{n}{2}} } \prod_{(i,j) \in E(H)} \mathtt h(\sigma_i,\sigma_j) \prod_{ (i,j) \in E(S) \setminus E(H) } \omega(\sigma_i,\sigma_j) \sqrt{\tfrac{\epsilon^2 \lambda}{n}} = \mathbf 1_{ S=\emptyset } \,.
\end{equation}
We now construct the solution $\{ u_{\sigma,H}: \sigma \in [k]^n, H \Subset \mathcal{K}_n, |E(H)| \leq D \}$ of \eqref{eq-linear-equation-simplified} as follows: let $u_{\sigma,H} = \frac{1}{k^{\frac{n}{2}}} \cdot \Xi(H)$, where 
\begin{equation}{\label{eq-def-Xi-intial}}
    \Xi(\emptyset)=1 \mbox{ and } \Xi(S)=0 \mbox{ for } \mathcal L(S) \neq \emptyset 
\end{equation}
and then iteratively define for all $\mathcal L(S)=\emptyset$ by 
\begin{equation}{\label{eq-def-Xi-itrative}}
\begin{aligned}
    \Xi(S)= & -\Big( \mathbb E_{\sigma\sim\nu}\Big[ \prod_{(i,j) \in E(S)} \mathtt h(\sigma_i,\sigma_j) \Big] \Big)^{-1} \sum_{ \substack{ H \subset S \\ \mathcal L(H)=\emptyset } } \big( \tfrac{\epsilon^2 \lambda}{n} \big)^{ \frac{|E(S)|-|E(H)|}{2} } \Xi(H) \\
    & \times\ \mathbb E_{\sigma\sim\nu}\Big[ \prod_{(i,j) \in E(H)} \mathtt h(\sigma_i,\sigma_j) \prod_{(i,j) \in E(S) \setminus E(H)} \omega(\sigma_i,\sigma_j) \Big] \,, 
\end{aligned}
\end{equation}
To prove Theorem~\ref{main-thm-SBM}, it suffices to show the following lemma, as incorporated in Section~\ref{subsec:proof-lem-4.4} of the appendix.
\begin{lemma}{\label{lem-bound-L2-norm}}
    The vector $\{ u_{\sigma,H}: \sigma \in [k]^n, H \Subset \mathcal{K}_n, |E(H)| \leq D \}$ satisfies \eqref{eq-linear-equation-simplified} and $\| u \|=O_{\delta,k}(1)$.
\end{lemma}

\noindent{\bf Acknowledgment.} The author thanks Jian Ding and Jingqiu Ding for helpful discussions, Alexander S. Wein for helpful comments on the revised low-degree conjecture, and Hang Du for pointing him to the reference \cite{SW25}.

\appendix

\section{Preliminary results in graphs and probability}{\label{sec:prelim}}

In this section we include some preliminary results we will use in the main part of the paper. We first state some results on the structural properties of graphs established in \cite{DDL23+, CDGL24+}.

\begin{lemma}[\cite{DDL23+}, Lemma~A.1]{\label{lemma-facts-graphs}}
Let $S,T\Subset \mathcal{K}_n$ satisfy $S\cong\mathbf{S}$ and $T\cong\mathbf{T}$ for some $\mathbf{S},\mathbf{T}\in \mathcal H$. Recall $S\Cup T,S\Cap T\Subset \mathcal{K}_n$ defined as edge-induced graphs of $\mathcal{K}_n$. We have the following hold:
\begin{enumerate}
    \item[(i)] $|V(S\cup T)|+|V(S\Cap T)|\le|V(S)|+|V(T)|$, and $|E(S\cup T)|+|E(S\Cap T)|=|E(S)|+|E(T)|$.
    \item[(ii)] $\Phi(S\cup T)\Phi(S\Cap T)\le\Phi(S)\Phi(T)$.
    \item[(iii)] If $\mathbf{S}\subset\mathbf{T}$, then $|\!\operatorname{Aut}(\mathbf{S})|\leq |\!\operatorname{Aut}(\mathbf{T})|\cdot |V(\mathbf{T})|^{ 2(|E(\mathbf{T})|-|E(\mathbf{S})|) }$.
    \item[(iv)] $\#\big\{ T'\Subset \mathcal{K}_n:S\Subset T',|V(T')|-|V(S)|=k,|E(T')|-|E(S)|=l \big\} \leq n^{k}(|V(S)|+k)^{2l}$.
    \item[(v)] $\#\big\{ T' \Subset S : |E(S)|-|E(T')|=k \big\} \le |E(S)|^k$.
\end{enumerate}
\end{lemma}

\begin{lemma}[\cite{CDGL24+}, Lemma~A.3]{\label{lem-decomposition-H-Subset-S}}
    For $H \subset S$, we can decompose $E(S)\setminus E(H)$ into $\mathtt m$ cycles ${C}_{\mathtt 1}, \ldots, {C}_{\mathtt m}$ and $\mathtt t$ paths ${P}_{\mathtt 1}, \ldots, {P}_{\mathtt t}$ for some $\mathtt m, \mathtt t\geq 0$ such that the following holds:
    \begin{enumerate}
        \item[(1)] ${C}_{\mathtt 1}, \ldots, {C}_{\mathtt m}$ are vertex-disjoint (i.e., $V(C_{\mathtt i}) \cap V(C_{\mathtt j})= \emptyset$ for all $\mathtt i \neq \mathtt j$) and $V(C_{\mathtt i}) \cap V(H)=\emptyset$ for all $1\leq\mathtt i\leq \mathtt m$.
        \item[(2)] $\operatorname{EndP}({P}_{\mathtt j}) \subset V(H) \cup (\cup_{\mathtt i=1}^{\mathtt m} V(C_{\mathtt i})) \cup (\cup_{\mathtt k=1}^{\mathtt j-1} V(P_{\mathtt k})) \cup \mathcal L(S)$ for $1 \leq \mathtt j \leq \mathtt t$.
        \item[(3)] $\big( V(P_{\mathtt j}) \setminus \operatorname{EndP}(P_{\mathtt j}) \big) \cap \big( V(H) \cup (\cup_{\mathtt i=1}^{\mathtt m} V(C_{\mathtt i})) \cup (\cup_{\mathtt k=1}^{\mathtt j-1} V(P_{\mathtt k}) ) \cup \mathcal L (S) \big) = \emptyset$ for $\mathtt 1 \leq \mathtt j \leq \mathtt t$.
        \item[(4)] $\mathtt t = |\mathcal L(S) \setminus V(H)|+\tau(S)-\tau(H)$.
    \end{enumerate}
\end{lemma}

\begin{lemma}[\cite{CDGL24+}, Corollary~A.4]{\label{lem-revised-decomposition-H-Subset-S}}
    For $H \subset S$, we can decompose $E(S)\setminus E(H)$ into $\mathtt m$ cycles ${C}_{\mathtt 1}, \ldots, {C}_{\mathtt m}$ and $\mathtt t$ paths ${P}_{\mathtt 1}, \ldots, {P}_{\mathtt t}$ for some $\mathtt m, \mathtt t\geq 0$ such that the following hold: 
    \begin{enumerate}
        \item[(1)] ${C}_{\mathtt 1}, \ldots, {C}_{\mathtt m}$ are independent cycles in $S$.
        \item[(2)] $V(P_{\mathtt j}) \cap \big( V(H) \cup (\cup_{\mathtt i=1}^{\mathtt m} V(C_{\mathtt i})) \cup (\cup_{\mathtt k \neq \mathtt j} V(P_{\mathtt k}) ) \cup \mathcal L (S) \big) = \operatorname{EndP}(P_{\mathtt j})$ for $1 \leq \mathtt j \leq \mathtt t$.
        \item[(3)] $\mathtt t \leq 5(|\mathcal L (S) \setminus V(H) | + \tau(S)-\tau(H))$.
    \end{enumerate}
\end{lemma}

\begin{lemma}[\cite{CDGL24+}, Lemma~A.9]{\label{lem-enu-Subset-large-graph}}
    For $H \subset \mathcal{K}_n$, we have (below we write $\mathfrak P=\{(p_{N+1},\ldots,p_D): \sum_{i=N+1}^{D} p_i \leq p, \text{ for all } N+1 \leq l\leq D \}$ and $\sum_\mathfrak P$ for the summation over $(p_{N+1}, \ldots, p_D)\in \mathfrak P$)
    \begin{align}
        & \#\Big\{ S \mbox{ admissible}\!: H \ltimes S; |\mathfrak C_l(S,H)|=0 \mbox{ for } l > N; |\mathcal L(S) \setminus V(H)| + \tau(S)-\tau(H) = m;  \nonumber \\
        & |E(S)|-|E(H)|=p, |V(S)|-|V(H)|=q, |E(S)| \leq D \Big\} \leq (2D)^{3m} n^{q} \sum_{\mathfrak P} \prod_{j=N+1}^{D} \frac{1}{p_j!} \,. \label{eq-enu-general-subset-large-graph}
    \end{align}   
\end{lemma}

\begin{lemma}[\cite{CDGL24+}, Lemma~A.10]{\label{lem-enu-Subset-small-graph}}
    For $S \subset \mathcal K_n$ with $|E(S)| \leq D$, we have
    \begin{equation}{\label{eq-enu-Subset-small-graph}}
    \begin{aligned}
        \# \Big\{ & H: H \ltimes S, |\mathcal L(S) \setminus V(H)| + \tau(S)-\tau(H) = m ,  \\
        & \mathfrak C_j(S;H) = m_j, N+1 \leq j \leq D \Big\} \leq D^{15m} \prod_{j=N+1}^{D} \binom{ |\mathcal C_j(S)| }{ m_j } \,.
    \end{aligned}
    \end{equation}
\end{lemma}

We now show some preliminary results on calculating expectations under $\sigma\sim\nu$.

\begin{lemma}{\label{lem-expectation-over-chain}}
    For a path $\mathcal{P}$ with $V(\mathcal P)= \{ v_0, \ldots, v_l \} $ and $\operatorname{EndP}(\mathcal P)=\{ v_0,v_l \}$, we have 
    \begin{equation}{\label{eq-expectation-over-chain}}
    \begin{aligned}
        \mathbb{E}_{\sigma \sim \nu} \Big[ \prod_{i=1}^{l} \Big( a+b\omega(\sigma_{i-1},\sigma_i) \Big) \mid \sigma_0, \sigma_l \Big] = a^l + b^l \cdot \omega(\sigma_0, \sigma_l)  \,.
    \end{aligned}
    \end{equation}
\end{lemma}
\begin{proof}
    By independence, we see that $\mathbb E_{\sigma \sim \nu} \big[ \prod_{i \in I} \omega(\sigma_{i-1},\sigma_i) \mid \sigma_0, \sigma_l \big] = 0$ if $I \subsetneq [l]$. Thus,
    \begin{align*}
        \mathbb{E}_{\sigma \sim \nu} \Big[ \prod_{i=1}^{l} \Big( a+b\omega(\sigma_{i-1},\sigma_i) \Big) \mid \sigma_0, \sigma_l \Big] = a^l + b^l \mathbb{E}_{\sigma \sim \nu}\Big[ \prod_{i=1}^{l} \omega(\sigma_{i-1},\sigma_i) \mid \sigma_0, \sigma_l \Big] \,.
    \end{align*}
    It remains to prove that
    \begin{equation}\label{eq-finalgoal1-sign-of-path-cycle-conditioned-on-endpoints}
        \mathbb{E}_{\sigma \sim \nu} \Big[ \prod_{i=1}^{l} \omega(\sigma_{i-1},\sigma_i) \mid \sigma_0, \sigma_l \Big] = \omega(\sigma_0, \sigma_l)\,.
    \end{equation}
    We shall show \eqref{eq-finalgoal1-sign-of-path-cycle-conditioned-on-endpoints} by induction. The case $l=1$ follows immediately. Now we assume that \eqref{eq-finalgoal1-sign-of-path-cycle-conditioned-on-endpoints} holds for $l$. Then we have
    \begin{align*}
        &\mathbb{E}_{\sigma \sim \nu} \Big[ \prod_{i=1}^{l+1} \omega(\sigma_{i-1},\sigma_i) \mid \sigma_0, \sigma_{l+1} \Big]\\
        = & \mathbb{E}_{\sigma \sim \nu}\Big[\omega(\sigma_l,\sigma_{l+1})\mathbb{E}_{\sigma \sim \nu} \Big[ \prod_{i=1}^{l} \omega(\sigma_{i-1},\sigma_i) \mid \sigma_0, \sigma_l,\sigma_{l+1} \Big] \mid \sigma_0,\sigma_{l+1}\Big]\\
        = & \mathbb{E}_{\sigma \sim \nu}\Big[ \omega(\sigma_l,\sigma_{l+1}) \omega(\sigma_0,\sigma_l) \mid \sigma_0,\sigma_{l+1}\Big] = \omega(\sigma_0,\sigma_{l+1})\,,
    \end{align*}
    which completes the induction procedure. 
\end{proof}

\begin{lemma}{\label{lem-leaf-cancellation}}
    For $H \subset S$ with $\mathcal L(S) \not\subset V(H)$, we have 
    \begin{equation}{\label{eq-leaf-cancellation}}
        \mathbb E_{\sigma\sim\nu}\Big[ \prod_{ (i,j) \in E(S) \setminus E(H) } \omega(\sigma_i,\sigma_j) \mid \{ \sigma_u: u \in V(H) \} \Big] = 0 \,.
    \end{equation}
\end{lemma}
\begin{proof}
    Denote $v \in \mathcal L(S) \setminus V(H)$ and $(v,w) \in E(S)$. Define $\sigma_{U}$ and $\sigma_{\setminus U}$ to be the restriction of $\sigma$ on $U$ and on $[n] \setminus U$, respectively. Also define $\nu_U$ and $\nu_{\setminus U}$ to be the restriction of $\nu$ on $U$ and on $[n] \setminus U$, respectively. Then we have (let $\mathtt V=V(H)$)
    \begin{align*}
        & \mathbb E_{\sigma\sim\nu}\Big[ \prod_{ (i,j) \in E(S) \setminus E(H) } \omega(\sigma_i,\sigma_j) \mid \{ \sigma_u: u \in V(H) \} \Big] \\
        =\ & \mathbb E_{ \sigma_{([n] \setminus \mathtt V)} \sim \nu_{([n] \setminus\mathtt V)} }\Big[ \prod_{ (i,j) \in E(S) \setminus E(H) } \omega(\sigma_i,\sigma_j) \Big] \\
        =\ & \mathbb E_{ \sigma_{([n] \setminus (\mathtt V \cup \{v\}))} \sim \nu_{([n] \setminus (\mathtt V \cup \{v\}))} }  \mathbb E_{ \sigma_{\{v\}} \sim \nu_{ \{v\} } }\Big[ \omega(\sigma_v,\sigma_w) \prod_{ (i,j) \in E(S) \setminus (E(H)\cup \{v,w\}) } \omega(\sigma_i,\sigma_j) \Big] = 0 \,. \qedhere
    \end{align*}
\end{proof}

\section{Supplementary proofs}{\label{sec:supp-proofs}}

\subsection{Proof of Lemma~\ref{lem-bound-low-deg-Adv-conditional}}{\label{subsec:proof-lem-3.5}}

In this subsection we provide the postponed proof of Lemma~\ref{lem-bound-low-deg-Adv-conditional}. For two graphs $S_1,S_2 \Subset \mathcal{K}_n$, define the polynomial $\phi_{S_1,S_2}$ associated with $S_1,S_2$ by 
\begin{equation}\label{eq-def-phi-S1,S2}
    \phi_{S_1,S_2}\big(\{A_{i,j}\},\{B_{i,j}\}\big)=\big(q(1-q)\big)^{-\frac{|E(S_1)|+|E(S_2)|}{2}}\prod_{(i,j)\in E(S_1)}\overline{A}_{i,j}\prod_{(i,j)\in E(S_2)}\overline{B}_{i,j} \,,
\end{equation}
where $\overline{A}_{i,j}=A_{i,j}-q,\overline{B}_{i,j}=B_{i,j}-q$ for all $(i,j)\in \operatorname{U}$. In particular, $\phi_{\emptyset,\emptyset}\equiv 1$. It can be easily checked that $\mathcal O_D=\{ \phi_{S_1,S_2}: |E(S_1)|+|E(S_2)| \leq D \}$ constitutes a standard orthogonal basis of $\mathcal P_{D}$ under $\Qb$. In addition, we say a polynomial \(\phi_{S_1,S_2}\in \mathcal O_D\) is \emph{admissible} if both \(S_1\) and \(S_2\) are admissible graphs. Furthermore, we define $\mathcal O_D'\subset \mathcal O_D$ as the set of admissible polynomials in $\mathcal O_D$, and define \(\mathcal P_{D}' \subset \mathcal P_{D}\) as the linear subspace spanned by polynomials in $\mathcal O_D'$. It has been shown in \cite[Proposition~3.4]{DDL23+} that for any $f \in \mathcal P_{D}$, there exists $f'\in \mathcal P_{D}'$ such that $\mathbb E_{\Qb}[(f')^2] \leq 8 \mathbb E_{\Qb}[f^2]$ and $f'=f$ a.s. under both $\overline{\Pb}$ and $\overline{\Pb}_*$. Thus, we get that
\begin{align}{\label{eq-reduce-to-admissible}}
    \mathsf{Adv}_{\leq D}\Big( \tfrac{ \mathrm{d}\overline{\Pb}(\cdot \mid \pi_*(1)=i) }{ \mathrm{d}\Qb } \Big) = \sup_{ f \in \mathcal P_{D} } \frac{ \mathbb E_{\overline{\Pb}}[ f \mid \pi_*(1)=1 ] }{ \sqrt{ \mathbb E_{\Qb}[f^2] } }  \leq 2\sqrt{2} \sup_{ f \in \mathcal P'_{D} } \frac{ \mathbb E_{\overline{\Pb}}[ f \mid \pi_*(1)=1 ] }{ \sqrt{ \mathbb E_{\Qb}[f^2] } } \,.
\end{align}
Thus, it suffices to show the right hand side of \eqref{eq-reduce-to-admissible} is bounded by $O_{\delta}(1)$. Similar as \eqref{eq-low-deg-Adv-trandsform}, we have
\begin{equation*}
    \sup_{ f \in \mathcal P'_{D} } \frac{ \mathbb E_{\overline{\Pb}}[ f \mid \pi_*(1)=1 ]^2 }{ \mathbb E_{\Qb}[f^2] } = \sum_{ \phi_{S_1,S_2} \in \mathcal O'_D } \mathbb E_{\overline{\Pb}} \big[ \phi_{S_1,S_2}  \mid \pi_*(1) = i \big]^2 \,.
\end{equation*}
Without loss of generality, in the following we will only show that
\begin{equation}{\label{eq-final-goal-conditional-low-degree-Adv}}
    \sum_{ \phi_{S_1,S_2} \in \mathcal O'_D } \mathbb E_{\overline{\Pb}} \big[ \phi_{S_1,S_2}  \mid \pi_*(1) = 1 \big]^2 = O_{\delta}(1) \,.
\end{equation}
For a deterministic permutation $\pi\in\mathfrak S_n$, we use \(\Pb_\pi\) and \(\overline{\Pb}_\pi\) to represent \(\Pb_*(\cdot\mid \pi_*=\pi)\) and \(\overline{\Pb}_*(\cdot\mid \pi_*=\pi)\), respectively. For $S_1,S_2 \Subset \mathcal{K}_n$ with $|E(S_1)|,|E(S_2)| \leq D$, define
\begin{equation}{\label{eq-def-mathtt-F}}
    \mathtt F(S_1,S_2) = \sum_{ \substack{ \mathbf H_0\in \mathcal H \\ \mathbf H_0\hookrightarrow S_i,i=1,2 } } n^{-\frac{|V(S_1)|+|V(S_2)|}{2}} \rho^{|E(\mathbf H_0)|} d^{-6( |E(S_1)|+|E(S_2)|-2|E(\mathbf H_0)| )} \operatorname{Aut}(\mathbf H_0) \,, 
\end{equation}
where the notation $\mathbf H\hookrightarrow S$ means that $\mathbf H$ can be embedded into $S$ as a subgraph. The key of our proof is the following estimation. 

\begin{proposition}{\label{prop-complicate-bound}}
    We have
    \begin{equation}{\label{eq-complicate-bound}}
        \big| \mathbb E_{\overline{\Pb}}[ \phi_{S_1,S_2} \mid \pi_*(1)=1 ] \big| \leq [1+o(1)] \big( \mathbf 1_{ 1 \not \in V(S_1) \cap V(S_2) } + n \mathbf 1_{ 1 \in V(S_1) \cap V(S_2) } \big) \cdot \mathtt F(S_1,S_2) \,.
    \end{equation}
\end{proposition}
\begin{proof}
    For $S_0 \hookrightarrow S_1,S_2$, define
    \begin{equation}{\label{eq-def-M-S-0,1,2}}
        M(S_0,S_1,S_2) = \rho^{|E(S_0)|} n^{ -\frac{|V(S_1)|+|V(S_2)|}{2}+|V(S_0)| } D^{-7( |E(S_1)|+|E(S_2)|-2|E(S_0)| )} \,.
    \end{equation}
    Using \cite[Equation~(3.29)]{DDL23+}, for any $S_1,S_2 \Subset \mathcal{K}_n$ with at most $D$ edges and any permutation $\pi\in\mathfrak S_n$, we have that (denote $S_0 = S_1 \cap S_2$ below)
    \begin{equation}{\label{eq-bound-conditional-moment}}
        \big| \mathbb E_{\overline{\Pb}_\pi}[ \phi_{S_1,S_2} ] \big| \leq M(S_0,S_1,S_2) \,.
    \end{equation}
    Thus, note that
    \begin{align}
        \big| \mathbb E_{\overline{\Pb}}[ \phi_{S_1,S_2} \mid \pi_*(1)=1 ] \big| &= \frac{1}{(n-1)!} \Big| \sum_{ \pi\in\mathfrak S_n,\pi(1)=1 } \mathbb E_{\overline{\Pb}_{\pi}}[ \phi_{S_1,S_2} ] \Big| \nonumber \\
        &\leq \frac{1}{(n-1)!} \sum_{ \pi\in\mathfrak S_n,\pi(1)=1 } \big| \mathbb E_{\overline{\Pb}_{\pi}}[ \phi_{S_1,S_2} ] \big| \,, \label{eq-conditional-moment-relax-1}
    \end{align}
    we can group the permutations $\pi\in\mathfrak S_n$ according to the realization of $S_0=S_1 \cap \pi^{-1}(S_2)$ and obtain that $\big| \mathbb E_{\overline{\Pb}}[ \phi_{S_1,S_2} \mid \pi_*(1)=1 ] \big|$ is bounded by
    \begin{equation}
        \frac{1}{(n-1)!} \sum_{ \substack{ S_0,S'_0: S_0 \cong S_0' \\ S_0 \subset S_1, S_0' \subset S_2 } } \#\{ \pi\in\mathfrak S_n: \pi(1)=1, \pi(S_0)=S_0', S_0 = S_1 \cap \pi^{-1}(S_2) \} M(S_0,S_1,S_2) \,. \label{eq-conditional-moment-relax-2}
    \end{equation}
    For the case when $1 \not\in V(S_1) \cap V(S_2)$, we must have $1 \not\in V(S_0)$ and thus
    \begin{equation}{\label{eq-def-enum}}
        \mathsf{Enum} := \#\{ \pi\in\mathfrak S_n: \pi(1)=1, \pi(S_0)=S_0' \}
    \end{equation}
    is upper-bounded by
    \begin{equation}{\label{eq-enum-bound}}
        \operatorname{Aut}(S_0)(n-1-|V(S_0)|)! = [1+o(1)] \operatorname{Aut}(S_0)(n-1)! \cdot n^{-|V(S_0)|} \,.
    \end{equation}
    Thus, we have that \eqref{eq-conditional-moment-relax-2} is further upper-bounded by (up to a factor of $(1+o(1))$)
    \begin{align*}
        &\sum_{\substack{S_0,S_0':S_0\cong S_0'\\S_0\subset S_1,S_0'\subset S_2}}n^{-|V(S_0)|} \operatorname{Aut}(S_0)M(S_0,S_1,S_2)\\
        =&\sum_{\substack{\mathbf H_0\in \mathcal H\\\mathbf H_0\xhookrightarrow{} S_i,i=1,2}} \frac{\operatorname{Aut}(\mathbf H_0)}{n^{|V(\mathbf H_0)|}} M(\mathbf H_0,S_1,S_2)\cdot \#\{(S_0,S_0'):S_0\subset S_1,S_0'\subset S_2,S_0\cong S_0'\cong \mathbf H_0\}\\
        \le&\sum_{\substack{\mathbf H_0\in \mathcal H\\\mathbf H_0\xhookrightarrow{} S_i,i=1,2}} \frac{\operatorname{Aut}(\mathbf H_0)}{n^{|V(\mathbf H_0)|}} M(\mathbf H_0,S_1,S_2)\cdot D^{|E(S_1)|+E(S_2)-2|E(\mathbf H_0)|}\,,
    \end{align*}
    where the last inequality follows from Lemma~\ref{lemma-facts-graphs} (v) and the assumption that $|E(S_1)|$ and $|E(S_2)|$ are bounded by $D$. For the case $1 \in V(S_1) \cap V(S_2)$, we have $\mathsf{Enum} \leq [1+o(1)] \operatorname{Aut}(S_0) n! \cdot n^{-|V(S_0)|}$, and thus we get an additional factor of $n$ in the final estimation. This completes the proof of Proposition~\ref{prop-complicate-bound}.
\end{proof}

We are now ready to establish \eqref{eq-final-goal-conditional-low-degree-Adv}, thus completing the proof of Lemma~\ref{lem-bound-low-deg-Adv-conditional}.

\begin{proof}[Proof of \eqref{eq-final-goal-conditional-low-degree-Adv}]
    We divide the left hand side of \eqref{eq-final-goal-conditional-low-degree-Adv} into two parts: the first part consists of those $S1,S2$ such that $1 \not\in V(S_1) \cap V(S_2)$, and the second part consists of those with $1 \in V(S_1) \cap V(S_2)$. We bound the first part
    \begin{align}
        \sum_{ \substack{ \phi_{S_1,S_2} \in \mathcal O'_D \\ 1 \not\in V(S_1) \cap V(S_2) } } \mathbb E_{\overline{\Pb}}\big[ \phi_{S_1,S_2} \mid \pi_*(1)=1 \big]^2 \label{eq-part-1}
    \end{align}
    via Proposition~\ref{prop-complicate-bound} by
    \begin{align}
        &\ [1+o(1)] \sum_{\substack{S_1,S_2\Subset\mathcal{K}_n \text{admissible}\\|E(S_1)|+|E(S_2)|\le D}} \Bigg( \sum_{\substack{\mathbf H_0\in\mathcal H\\\mathbf H_0\xhookrightarrow{} S_i,i=1,2}} \frac{ \rho^{|E(\mathbf H_0)|} \operatorname{Aut}(\mathbf H_0) }{ D^{6(|E(S_1)|+|E(S_2)|-2|E(\mathbf H_0)|)} n^{\frac{1}{2}(|V(S_1)|+|V(S_2)|)}  } \Bigg)^2 \nonumber \\
        \leq &\ [1+o(1)] \sum_{\substack{S_1,S_2\Subset \mathcal{K}_n\text{admissible}\\|E(S_1)|+|E(S_2)|\le D}} \Bigg( \sum_{\substack{\mathbf H_0\in\mathcal H\\\mathbf H_0\xhookrightarrow{} S_i,i=1,2}} \frac{ \rho^{2|E(\mathbf H_0)|} \operatorname{Aut}(\mathbf H_0)^2 }{ D^{6(|E(S_1)|+|E(S_2)|-2|E(\mathbf H_0)|)} n^{|V(S_1)|+|V(S_2)|}  } \Bigg) \nonumber \\
        & \times \Bigg( \sum_{\substack{\mathbf H_0\in\mathcal H\\\mathbf H_0\xhookrightarrow{} S_i,i=1,2}} D^{-6(|E(S_1)|+|E(S_2)|-2|E(\mathbf H_0)|)} \Bigg) \,, \label{eq-part-1-second-bracket}
    \end{align}
    where the last inequality comes form Cauchy-Schwartz inequality. It is straightforward to check that when $|E(S_1)|+|E(S_2)|\le d$,
    \begin{align}
        & \sum_{\substack{\mathbf H_0\in\mathcal H\\\mathbf H_0\xhookrightarrow{} S_i,i=1,2}}d^{-6(|E(S_1)|+|E(S_2)|-2|E(\mathbf H_0)|)} \nonumber \\
        \le &\ \Big(\sum_{\substack{\mathbf H_0\in\mathcal H\\\mathbf H_0\xhookrightarrow{} S_1}}d^{-6(|E(S_1)|-|E(\mathbf H_0)|)}\Big)^2 \le \Big(\sum_{k=0}^d d^{-6k}\cdot d^k\Big)^2\leq 2 \,, \label{eq-xxx}
    \end{align}
    where the second inequality follows from Lemma~\ref{lemma-facts-graphs} (v). Plugging \eqref{eq-xxx} into \eqref{eq-part-1}, we get that \eqref{eq-part-1} is bounded by a constant factor times
    \begin{align} 
        \sum_{\substack{\mathbf H_0\in\mathcal H\\\mathbf H_0\xhookrightarrow{} S_i,i=1,2}} \rho^{2|E(\mathbf H_0)|} \operatorname{Aut}(\mathbf H_0)^2 \sum_{\substack{S_1,S_2\Subset\mathcal{K}_n \text{admissible}\\|E(S_1)|+|E(S_2)|\le D}} \frac{n^{-|V(S_1)|-|V(S_2)|}}{ D^{6(|E(S_1)|+|E(S_2)|-2|E(\mathbf H_0)|)} } \,. \label{eq-A.13}
    \end{align}
    By denoting $S_1 \cong \mathbf{S}_1$ and $S_2 \cong \mathbf{S}_2$ with $\mathbf S_1,\mathbf S_2\in \mathcal H$, the right hand side of \eqref{eq-A.13} reduces to
    \begin{align}
        & \sum_{ \mathbf{H}_0 \in \mathcal H} \rho^{2|E(\mathbf H_0)|} \operatorname{Aut}(\mathbf H_0)^2 \sum_{ \substack{ \mathbf S_1,\mathbf S_2\in \mathcal H\text{ admissible}\\\mathbf H_0\xhookrightarrow{} \mathbf{S}_i,i=1,2 \\ |E(\mathbf S_1)|+|E(\mathbf S_2)|\le d } } \frac{ \# \{ S_1,S_2 : S_1 \cong \mathbf{S}_1, S_2 \cong \mathbf{S}_2 \} }{ D^{6(|E(\mathbf S_1)|+|E(\mathbf S_2)|-2|E(\mathbf H_0)|)} n^{|V(\mathbf S_1)|+|V(\mathbf S_2)|} } \nonumber \\
        \leq & \sum_{ \mathbf{H}_0 \in \mathcal H\operatorname{admissible} } \rho^{2|E(\mathbf H_0)|} \sum_{ \substack{ \mathbf S_1,\mathbf S_2\in \mathcal H,\mathbf H_0\xhookrightarrow{} \mathbf{S}_i,i=1,2 \\ |E(\mathbf S_1)|+|E(\mathbf S_2)|\le d } } \frac{ \operatorname{Aut}(\mathbf H_0)^2 }{ \operatorname{Aut}(\mathbf S_1) \operatorname{Aut}(\mathbf S_2) D^{ 6(|E(\mathbf S_1)|+|E(\mathbf S_2)|-2|E(\mathbf H_0)|) } } \,, \label{eq-A.13-relax}
    \end{align}
    where the inequality follows from
    \begin{align*}
        \# \{ S_1,S_2 : S_1 \cong \mathbf{S}_1, S_2 \cong \mathbf{S}_2 \} \leq [1+o(1)] \cdot \frac{ n^{|V(\mathbf S_1)|} }{ \operatorname{Aut}(\mathbf S_1) } \cdot \frac{ n^{|V(\mathbf S_2)|} }{ \operatorname{Aut}(\mathbf S_2) } \,.
    \end{align*}
    Using Lemma~\ref{lemma-facts-graphs} (iii), we have $\frac{ \operatorname{Aut}(\mathbf H_0)^2 }{ \operatorname{Aut}(\mathbf S_1) \operatorname{Aut}(\mathbf S_2) } \leq d^{ 2(|E(\mathbf S_1)|+|E(\mathbf S_2)|-2|E(\mathbf H_0)|) }$, and thus the right hand side of \eqref{eq-A.13-relax} is upper-bounded by
    \begin{equation}
    \begin{aligned}
        &\ \sum_{ \mathbf{H}_0\in \mathcal H \operatorname{admissible} } \rho^{2|E(\mathbf H_0)|} \sum_{ \substack{ \mathbf S_1,\mathbf S_2\subset \mathcal H, \mathbf H_0\xhookrightarrow{} \mathbf{S}_i,i=1,2 \\ |E(\mathbf S_1)|+|E(\mathbf S_2)|\le d } } d^{ -4(|E(\mathbf S_1)|+|E(\mathbf S_2)|-2|E(\mathbf H_0)|) } \\
        \leq&\ 4 \sum_{ \substack{\mathbf{H}_0\in \mathcal H \operatorname{admissible}\\ |E(\mathbf H_0)|\le d} } \rho^{2|E(\mathbf H_0)|} \,,
    \end{aligned}
    \end{equation}
    where we used the fact that (similar to \eqref{eq-xxx})
    \begin{align*}
        &\ \sum_{ \substack{ \mathbf S_1,\mathbf S_2\in \mathcal H, \mathbf H_0\xhookrightarrow{} \mathbf{S}_i,i=1,2 \\ |E(\mathbf S_1)|+|E(\mathbf S_2)|\le d } } d^{ -4(|E(\mathbf S_1)|+|E(\mathbf S_2)|-2|E(\mathbf H_0)|) }\\
        \le&\  \Bigg(\sum_{\substack{\mathbf H_0\xhookrightarrow{} \mathbf S,|E(\mathbf S)|\le d}}d^{-4(|E(\mathbf S)|-|E(\mathbf H)|)}\Bigg)^2\le \Bigg(\sum_{k=0}^dd^{-4k}\cdot (2d)^{2k}\Bigg)^2\le 4\,.
    \end{align*}
    In conclusion, we have shown that
    \begin{align}
        \eqref{eq-A.13-relax} \leq O_{\delta}(1) \cdot \sum_{ \mathbf{H}_0 \in \mathcal H\operatorname{admissible} } \rho^{2|E(\mathbf H_0)|} \rho^{2|E(\mathbf H_0)|} = O_{\delta}(1) \,,
    \end{align}
    where the equality follows from \cite[Lemma~A.3]{DDL23+}. Thus, we have
    \begin{equation}{\label{eq-part-1-bound}}
        \eqref{eq-part-1} = O_{\delta}(1) \,.
    \end{equation}
    We now deal with the second part
    \begin{equation}{\label{eq-part-2}}
        \sum_{ \substack{ \phi_{S_1,S_2} \in \mathcal O'_D \\ 1 \in V(S_1) \cap V(S_2) } } \mathbb E_{\overline{\Pb}}\big[ \phi_{S_1,S_2} \mid \pi_*(1)=1 \big]^2  \,.
    \end{equation}
    Following a similar derivation of \eqref{eq-A.13}, we see that \eqref{eq-part-2} is bounded by a constant factor times
    \begin{align*}
        \sum_{\substack{\mathbf H_0\in\mathcal H\\\mathbf H_0\xhookrightarrow{} S_i,i=1,2}} \rho^{2|E(\mathbf H_0)|} \operatorname{Aut}(\mathbf H_0)^2 \sum_{\substack{S_1,S_2\Subset\mathcal{K}_n \text{admissible} \\ 1\in V(S_1) \cap V(S_2) \\ |E(S_1)|+|E(S_2)|\le D}} \frac{n^{-|V(S_1)|-|V(S_2)|}}{D^{6(|E(S_1)|+|E(S_2)|-2|E(\mathbf H_0)|)}} \times n^2
    \end{align*}
    (Note that in the above there is an additional factor of $n^2$ compared with \eqref{eq-A.13}.) However, due to the restriction $1 \in V(S_1) \cap V(S_2)$, for any fixed $\mathbf S_1,\mathbf S_2 \in \mathcal H$  the enumeration of $(S_1,S_2)$ such that $S_1 \cong \mathbf S_1, S_2 \cong \mathbf S_2$ is upper-bounded by
    \begin{align*}
        \frac{ n^{|V(\mathbf S_1)|-1} }{ \operatorname{Aut}(\mathbf S_1) } \cdot \frac{ n^{|V(\mathbf S_2)|-1} }{ \operatorname{Aut}(\mathbf S_2) } = \frac{ n^{|V(\mathbf S_1)|+|V(\mathbf S_2)|} }{ \operatorname{Aut}(\mathbf S_1) \operatorname{Aut}(\mathbf S_2) } \times n^{-2} \,.
    \end{align*}
    Thus, after canceling the factor of $n^{-2}$ with the factor of $n^2$ above we derive the same upper bound as we obtain on \eqref{eq-part-1}. This shows that
    \begin{equation}{\label{eq-part-2-bound}}
        \eqref{eq-part-2} = O_{\delta}(1) \,.
    \end{equation}
    Combining \eqref{eq-part-1-bound} and \eqref{eq-part-2-bound} leads to the desired result.
\end{proof}

\subsection{Proof of Lemma~\ref{lem-bound-low-deg-Adv-conditional-SBM}}{\label{subsec:proof-lem-3.9}}

Now it remains to prove Lemma~\ref{lem-bound-low-deg-Adv-conditional-SBM}. Recall the definition of $\phi_{S_1,S_2}$ in \eqref{eq-def-phi-S1,S2} (here we will take $q=\tfrac{\lambda s}{n}$). We first show the following estimation. Again, we say a polynomial $\phi_{S_1,S_2} \in \mathcal O_D$ is admissible if both $S_1$ and $S_2$ are admissible graphs. Furthermore, we define $\mathcal O'_D \subset \mathcal O_D$ as the set of admissible polynomials in $\mathcal O_D$, and define $\mathcal P'_D \subset \mathcal P_D$ as the linear subspace spanned by polynomials in $\mathcal O'_D$. Using \cite[Proposition~4.5]{CDGL24+}, we see that for any $f\in \mathcal P_{D}$, there exists some $f'\in \mathcal P_{D}'$ such that $\mathbb E_{\Qb}[(f')^2]\le O_{\delta,k}(1)\cdot \mathbb{E}_{\Qb}[f^2]$ and $f'=f$ a.s. under both $\widetilde{\Pb}_*$ and $\widetilde\Pb$. Thus, we get that
\begin{align}
    \mathsf{Adv}_{\leq D}\Big( \tfrac{ \mathrm{d}\widetilde{\Pb}(\cdot \mid \pi_*(i)=j) }{ \mathrm{d}\Qb } \Big) = \sup_{ f \in \mathcal P_{D} } \frac{ \mathbb E_{\widetilde{\Pb}}[ f \mid \pi_*(i)=j ] }{ \sqrt{ \mathbb E_{\Qb}[f^2] } } \leq O_{\delta,k}(1) \cdot \sup_{ f \in \mathcal P'_{D} } \frac{ \mathbb E_{\widetilde{\Pb}}[ f \mid \pi_*(i)=j ] }{ \sqrt{ \mathbb E_{\Qb}[f^2] } } \,. \nonumber 
\end{align}
Thus, it suffices to show that
\begin{align*}
    \sup_{ f \in \mathcal P'_{D} } \frac{ \mathbb E_{\widetilde{\Pb}}[ f \mid \pi_*(1)=1 ] }{ \sqrt{ \mathbb E_{\Qb}[f^2] } } = \sum_{ \phi_{S_1,S_2} \in \mathcal O'_D } \mathbb E_{ \widetilde{\Pb} } \big[ \phi_{S_1,S_2} \mid \pi_*(i)=j \big]^2 = O_{\delta,k}(1) \,.
\end{align*}
For notational simplicity, in the following we will only show 
\begin{equation}{\label{eq-final-goal-cor-SBM}}
    \sum_{ \phi_{S_1,S_2} \in \mathcal O'_D } \mathbb E_{ \widetilde{\Pb} } \big[ \phi_{S_1,S_2} \mid \pi_*(1)=1 \big]^2 = O_{\delta,k}(1) \,,
\end{equation}
and the general cases for $\pi_*(i)=j$ can be derived in a similar manner. For $H \subset S$, we define $\mathtt N(S,H)$ to be 
\begin{align}
    \mathtt N(S,H) &= \big( \tfrac{D^{28}}{n^{0.1}} \big)^{ \frac{1}{2} (|\mathcal L(S) \setminus V(H)| + \tau(S)-\tau(H)) } (1-\tfrac{\delta}{2})^{|E(S)|-|E(H)|} \,. \label{eq-def-mathtt-N}
\end{align}
We first show the following estimation.
\begin{proposition}{\label{prop-untruncate-expectation-Pb}}
   For all admissible $S_1,S_2\Subset \mathcal K_n$ with $|E(S_1)|,|E(S_2)| \leq D$, we have that $\big| \mathbb{E}_{\widetilde\Pb}[\phi_{S_1,S_2}\mid\pi_*(1)=1] \big|$ is bounded by $O_{\delta,k}(1)$ times (note that in the summation below $H_1,H_2$ may have isolated vertices)
    \begin{align}{\label{eq-bound-expectation-Pb}}
         \sum_{ \substack{ H_1 \subset S_1, H_2 \subset S_2 \\ H_1 \cong H_2 } } \frac{ (\sqrt{\alpha}-\tfrac{\delta}{2})^{|E(H_1)|}  \operatorname{Aut}(H_1) }{ n^{|V(H_1)|-\mathbf{1}_{\{1 \in V(H_1 \cap V(H_2))\}}} }  * \frac{ 2^{|\mathfrak{C}(S_1,H_1)|+|\mathfrak{C}(S_2,H_2)|} \mathtt N(S_1,H_1) \mathtt N(S_2,H_2) }{ n^{\frac{1}{2}(|V(S_1)|+|V(S_2)|-|V(H_1)|-|V(H_2)|)} }  \,.
    \end{align}
\end{proposition}
\begin{proof}
    For each deterministic permutation $\pi \in \mathfrak S_n$ and each labeling $\sigma \in [k]^{n}$, we denote $\widetilde{\Pb}_{\sigma,\pi}=\widetilde{\Pb}( \cdot \mid \sigma_{*}=\sigma, \pi_*=\pi )$, $\widetilde{\Pb}_{\pi}=\widetilde{\Pb}( \cdot \mid \pi_*=\pi )$ and $\widetilde{\Pb}_{\sigma}= \widetilde{\Pb}( \cdot \mid \sigma_*=\sigma)$ respectively. It is clear that 
    \begin{align}
        \Big| \mathbb{E}_{\widetilde\Pb}\big[ \phi_{S_1,S_2} \mid \pi_*(1)=1 \big] \Big| &= \Big|\frac{1}{(n-1)!} \sum_{ \substack{ \pi\in\mathfrak{S}_n \\ \pi(1)=1 } } \mathbb{E}_{\widetilde\Pb_\pi}[\phi_{S_1,S_2}]\Big| \nonumber \\ 
        &\leq \frac{1}{(n-1)!} \sum_{ \substack{ \pi\in\mathfrak{S}_n \\ \pi(1)=1 } } \big| \mathbb{E}_{\widetilde\Pb_\pi}[\phi_{S_1,S_2}] \big| \,. \label{eq-relaxation-1} 
    \end{align}
    For $H\subset S$, we define
    \begin{equation}{\label{eq-def-mathtt-M}}
        \mathtt M(S,H) = \big( \tfrac{D^{8}}{n^{0.1}} \big)^{ \frac{1}{2}( |\mathcal L(S)\setminus V(H)| +\tau(S)-\tau(H) ) } (1-\tfrac{\delta}{2})^{|E(S)|-|E(H)|} \,.
    \end{equation}
    Applying \cite[Lemma~4.10]{CDGL24+}, we have that for all admissible $S_1,S_2\Subset \mathcal K_n$ with at most $D$ edges and for all permutation $\pi$ on $[n]$, denote $H_1 = S_1 \cap \pi^{-1}(S_2)$ and $H_2=\pi(S_1)\cap S_2$. We have that $\big|\mathbb E_{\widetilde\Pb_\pi}[\phi_{S_1,S_2}] \big|$ is bounded by $O_{\delta,k}(1)$ times
    \begin{equation}
    \begin{aligned}
        & (\sqrt{\alpha}-\tfrac{\delta}{2})^{|E(H_1)|} \sum_{ H_1 \ltimes K_1 \subset S_1 } \sum_{ H_2 \ltimes K_2 \subset S_2 } \frac{ \mathtt M(S_1,K_1) \mathtt M(S_2,K_2) \mathtt M(K_1,H_1) \mathtt M(K_2,H_2) }{ n^{ \frac{1}{2}(|V(S_1)|+|V(S_2)|-|V(H_1)|-|V(H_2)|) } }  \,.
    \end{aligned}
    \end{equation}
    Note that we have 
    \begin{equation}{\label{eq-measure-intersection-pattern}}
        \mu\big( \{ \pi: \pi(1)=1, \pi(H_1)=H_2 \} \big) \leq [1+o(1)] \cdot \operatorname{Aut}(H_1) n^{-|V(H_1)|-1+\mathbf 1_{1 \in V(H_1) \cap V(H_2)}} \,. 
    \end{equation}
    Thus, it yields that the right-hand side of \eqref{eq-relaxation-1} is bounded by (up to a $O_{\delta,k}(1)$ factor)
    \begin{align}
        & \sum_{\substack{ H_1,H_2: H_1 \cong H_2 \\ H_1 \subset S_1, H_2 \subset S_2 }} \frac{ \operatorname{Aut}(H_1) (\sqrt{\alpha} -\tfrac{\delta}{2})^{|E(H_1)|} }{ n^{|V(H_1)|-\mathbf 1_{1 \in V(H_1) \cap V(H_2)}} } * \frac{ \mathtt P(S_1,H_1) \mathtt P(S_2,H_2) }{ n^{ \frac{1}{2}(|V(S_1)|+|V(S_2)|-|V(H_1)|-|V(H_2)|) } } \,,  \label{eq-relaxation-2}
    \end{align}
    where (recall \eqref{eq-def-mathtt-M})
    \begin{equation}{\label{eq-def-mathtt-P-SBM}}
        \mathtt P(S,H) = \sum_{H \ltimes K \subset S} \mathtt M(S,K) \mathtt M(K,H) \,.
    \end{equation}
    Using \cite[Claim~4.11]{CDGL24+}, we have that $\mathtt P(S,H) \leq [1+o(1)] \cdot 2^{|\mathfrak{C}(S,H)|} \mathtt N(S,H)$, thus leading to the desired result.
\end{proof}

Now we prove \eqref{eq-final-goal-cor-SBM} formally, thus finishing the proof of Lemma~\ref{lem-bound-low-deg-Adv-conditional-SBM}.

\begin{proof}[Proof of \eqref{eq-final-goal-cor-SBM}]
    By Proposition~\ref{prop-untruncate-expectation-Pb} and Cauchy-Schwartz inequality, \eqref{eq-final-goal-cor-SBM} is upper-bounded by $O_{\delta,k}(1)$ times
    \begin{align}
        & \sum_{\phi_{S_1,S_2} \in \mathcal O'_D} \Big( \sum_{ \substack{ H_1 \subset S_1, H_2 \subset S_2 \\ H_1 \cong H_2 } } n^{0.02|\mathcal I(H_1)|} \mathtt N(S_1,H_1) \mathtt N(S_2,H_2) \Big)  \label{eq-L-leq-D-relaxation-1.2} \times \\
        & \Big( \sum_{ \substack{ H_1 \subset S_1, H_2 \subset S_2 \\ H_1 \cong H_2 } } \frac{ (\sqrt{\alpha}-\tfrac{\delta}{2})^{2|E(H_1)|} \operatorname{Aut}(H_1)^2 }{ n^{2|V(H_1)| - 2\cdot\mathbf 1_{1 \in V(H_1) \cap V(H_2)}+ 0.02|\mathcal I(H_1)| } } \cdot \frac{ 4^{|\mathfrak{C}(S_1,H_1)|+|\mathfrak{C}(S_2,H_2)|} \mathtt N(S_1,H_1) \mathtt N(S_2,H_2) }{ n^{( |V(S_1)|+|V(S_2)|-|V(H_1)|-|V(H_2)| )} } \Big) \,. \nonumber 
    \end{align}
    It was shown in \cite[Equation~(4.35)]{CDGL24+} that the first bracket in \eqref{eq-L-leq-D-relaxation-1.2} is bounded by $O(1)$. Thus, we see that the left hand side of \eqref{eq-final-goal-cor-SBM} is bounded by $O_{\delta,k}(1)$ times
    \begin{align*}
        \sum_{\substack{H_1 \cong H_2, H_1,H_2 \text{ admissible} \\ |E(H_1)|+|E(H_2)| \leq D}} & \frac{ (\sqrt{\alpha}-\tfrac{\delta}{2})^{2|E(H_1)|} \operatorname{Aut}(H_1)^2 }{ n^{2|V(H_1)| - 2\cdot\mathbf 1_{1 \in V(H_1) \cap V(H_2)}+ 0.02|\mathcal I(H_1)| } } \cdot \mathtt L(H_1,H_2) \,,
    \end{align*}
    where
    \begin{align*}
        \mathtt L(H_1,H_2) = \sum_{ \substack{ H_1 \subset S_1, H_2 \subset S_2 } } \frac{ 4^{|\mathfrak{C}(S_1,H_1)|+|\mathfrak{C}(S_2,H_2)|} \mathtt N(S_1,H_1) \mathtt N(S_2,H_2) }{ n^{( |V(S_1)|+|V(S_2)|-|V(H_1)|-|V(H_2)| )} } \,.
    \end{align*}
    Using \cite[Equation~(4.37)]{CDGL24+}, we see that $\mathtt L(H_1,H_2)=O_{\delta,k}(1)$. Thus, the left hand side of \eqref{eq-final-goal-cor-SBM} is bounded by $O_{\delta,k}(1)$ times
    \begin{align}
        & \sum_{\substack{H_1 \cong H_2, H_1,H_2 \text{ admissible} \\ |E(H_1)|+|E(H_2)| \leq D}} \frac{ (\sqrt{\alpha}-\tfrac{\delta}{2})^{2|E(H_1)|} \operatorname{Aut}(H_1)^2 }{ n^{2|V(H_1)| - 2\cdot\mathbf 1_{1 \in V(H_1) \cap V(H_2)}+ 0.02|\mathcal I(H_1)| } } \nonumber \\
        =\ & \sum_{\substack{H_1 \cong H_2, H_1,H_2 \text{ admissible} \\ |E(H_1)|+|E(H_2)| \leq D\\ 1 \not \in V(H_1) \cap V(H_2) }} \frac{ (\sqrt{\alpha}-\tfrac{\delta}{2})^{2|E(H_1)|} \operatorname{Aut}(H_1)^2 }{ n^{2|V(H_1)| + 0.02|\mathcal I(H_1)| } }  \label{eq-cor-SBM-part-1-final} \\
        + & \sum_{\substack{H_1 \cong H_2, H_1,H_2 \text{ admissible} \\ |E(H_1)|+|E(H_2)| \leq D\\ 1 \in V(H_1) \cap V(H_2) }} \frac{ (\sqrt{\alpha}-\tfrac{\delta}{2})^{2|E(H_1)|} \operatorname{Aut}(H_1)^2 }{ n^{2|V(H_1)| + 0.02|\mathcal I(H_1)| } } \times n^{2} \,. \label{eq-cor-SBM-part-2-final}
    \end{align}
    Recall that we use $\widetilde{H}_1$ to denote the subgraph of $H_1$ obtained by removing all the vertices in $\mathcal I(H_1)$. In addition, for $|V(H_1)| \leq |V(S_1)| \le 2D$, we have $\operatorname{Aut}(H_1) = \operatorname{Aut}( \widetilde{H}_1 ) \cdot |\mathcal I(H_1)|! \leq (2D)^{|\mathcal I(H_1)|} \operatorname{Aut}( \widetilde{H}_1 )$. Thus, we see that
    Thus, we have that 
    \begin{align*}
        \eqref{eq-cor-SBM-part-1-final} &\leq \sum_{\substack{|E(\mathbf H)| \leq D, \mathcal I(\mathbf H)=\emptyset \\ \mathbf H \textup{ is admissible} }} \sum_{j \geq 0}\ \sum_{ \substack{ (H_1,H_2): \widetilde{H}_1 \cong \widetilde{H}_2 \cong \mathbf H \\ |\mathcal I(H_1)|=|\mathcal I(H_2)|=j } } n^{-0.01j} \cdot \frac{ \mathrm{Aut}(\mathbf H)^2 (2D)^{2j} (\sqrt{\alpha}-\tfrac{\delta}{4})^{2 |E(\mathbf H)| } }{ n^{2(|V(\mathbf H)|+j)} } \\
        &\circeq \sum_{\substack{|E(\mathbf H)| \leq D, \mathcal I(\mathbf H)=\emptyset \\ \mathbf H \textup{ is admissible} }} (\sqrt{\alpha}-\tfrac{\delta}{4})^{2 |E(\mathbf H)| } \leq O_{\delta,k}(1) \,,
    \end{align*}
    As for \eqref{eq-cor-SBM-part-2-final}, due to the restriction $1 \in V(H_1) \cap V(H_2)$, for any fixed $\mathbf H_1,\mathbf H_2 \in \mathcal H$  the enumeration of $(H_1,H_2)$ such that $H_1 \cong \mathbf H_1, H_2 \cong \mathbf H_2$ is upper-bounded by
    \begin{align*}
        \frac{ n^{|V(\mathbf H_1)|-1} }{ \operatorname{Aut}(\mathbf H_1) } \cdot \frac{ n^{|V(\mathbf H_2)|-1} }{ \operatorname{Aut}(\mathbf H_2) } = \frac{ n^{|V(\mathbf H_1)|+|V(\mathbf H_2)|} }{ \operatorname{Aut}(\mathbf H_1) \operatorname{Aut}(\mathbf H_2) } \times n^{-2} \,.
    \end{align*}
    Thus, after canceling the factor of $n^{-2}$ with the factor of $n^2$ above we derive the same upper bound as we obtain on \eqref{eq-cor-SBM-part-2-final}. This yields that $\eqref{eq-cor-SBM-part-2-final}=O_{\delta,k}(1)$ and thus completes our proof.
\end{proof}

\subsection{Proof of Lemma~\ref{lem-bound-L2-norm}}{\label{subsec:proof-lem-4.4}}

In this subsection we present the proof of Lemma~\ref{lem-bound-L2-norm}. The crucial input in our proof of Lemma~\ref{lem-bound-L2-norm} is the following estimation on $\Xi(H)$. 

\begin{lemma}{\label{lem-bound-Xi}}
    Assuming \eqref{eq-assumption-parameter} with $|E(S)|\leq D$ and $\mathcal L(S)=\emptyset$. We have the following estimation:
    \begin{enumerate}
        \item[(1)] If $S=S_1 \cup S_2$ such that $V(S_1) \cap V(S_2)=\emptyset$, then $\Xi(S)=\Xi(S_1) \Xi(S_2)$.
        \item[(2)] If $S$ can be decomposed into different cycles $C_1,\ldots,C_m$, then 
        \begin{align*}
            |\Xi(S)| \leq k^m (1-\delta/2)^{\frac{|E(H)|}{2}} n^{-\frac{|E(H)|}{2}} \,.
        \end{align*}
        \item[(3)] If $\tau(S)>0$ and $\mathfrak C_k(S)=m_k$ for $3 \leq k \leq D$, then
        \begin{align*}
            |\Xi(S)| \leq (10\tau(S))! \cdot (2kD)^{10\tau(H)} (1-\delta/2)^{\frac{|E(H)|}{2}} n^{-\frac{|E(H)|}{2}} \,.
        \end{align*}
    \end{enumerate}
\end{lemma}
\begin{proof}
    We first prove Item~(1). We will prove by induction on $|E(S)|$. The case $S=\emptyset$ is trivial. Now Suppose $S=S_1 \cup S_2$, for all $H \subset S$ with $\mathcal L(H)=\emptyset$, we have $H=H_1 \cup H_2$ with $H_1 \subset S_1,H_2 \subset S_2$ and $\mathcal L(H_1),\mathcal L(H_2)=\emptyset$. In addition, denote 
    \begin{align*}
        & \mathtt P(S) = \mathbb E_{\sigma\sim\nu}\Big[ \prod_{(i,j) \in E(S)} \mathtt h(\sigma_i,\sigma_j) \Big] \,, \\
        & \mathtt Q(S,H)= \mathbb E_{\sigma\sim\nu}\Big[ \prod_{(i,j) \in E(H)} \mathtt h(\sigma_i,\sigma_j) \prod_{(i,j) \in E(S) \setminus E(H)} \omega(\sigma_i,\sigma_j) \Big] \,.
    \end{align*}
    We have $\mathtt P(S)=\mathtt P(S_1)\mathtt P(S_2)$ and $\mathtt Q(S,H)=\mathtt Q(S_1,H_1) \mathtt Q(S_2,H_2)$. Also, it is straightforward to verify that $\mathtt Q(S,S)=\mathtt P(S)$. Thus, we have 
    \begin{align*}
        -\Xi(S) =& \sum_{ \substack{ H_1 \subsetneq S_1, H_2 \subsetneq S_2 \\ \mathcal L(H_1), \mathcal L(H_2)=\emptyset } } \frac{ \mathtt Q(S_1,H_1) \mathtt Q(S_2,H_2) \Xi(H_1 \sqcup H_2) }{ \mathtt P(S_1)\mathtt P(S_2) } + \sum_{ \substack{ H_2 \subsetneq S_2 \\ \mathcal L(H_2)=\emptyset } } \frac{ \mathtt Q(S_2,H_2) \Xi(S_1 \sqcup H_2) }{ \mathtt P(S_2) } \\
        & + \sum_{ \substack{ H_1 \subsetneq S_1 \\ \mathcal L(H_1) = \emptyset } } \frac{ \mathtt Q(S_1,H_1) \Xi(H_1 \sqcup S_2) }{ \mathtt P(S_1) } \\
        = & \sum_{ \substack{ H_1 \subsetneq S_1, H_2 \subsetneq S_2 \\ \mathcal L(H_1), \mathcal L(H_2)=\emptyset } } \frac{ \mathtt Q(S_1,H_1) \mathtt Q(S_2,H_2) \Xi(H_1)\Xi(H_2) }{ \mathtt P(S_1)\mathtt P(S_2) } + \sum_{ \substack{ H_2 \subsetneq S_2 \\ \mathcal L(H_2)=\emptyset } } \frac{ \mathtt Q(S_2,H_2) \Xi(S_1) \Xi(H_2) }{ \mathtt P(S_2) } \\
        & + \sum_{ \substack{ H_1 \subsetneq S_1 \\ \mathcal L(H_1) = \emptyset } } \frac{ \mathtt Q(S_1,H_1) \Xi(H_1)\Xi(S_2) }{ \mathtt P(S_1) } \\
        = & \ (-\Xi(S_1)) \cdot (-\Xi(S_2)) + \Xi(S_1) \cdot (-\Xi(S_2)) + (-\Xi(S_1)) \cdot \Xi(S_2) = - \Xi(S_1)\Xi(S_2) \,,
    \end{align*}
    where the second equality follows from induction hypothesis and the third equality follows from \eqref{eq-def-Xi-itrative}. This yields Item~(1).

    Now we focus on Item~(2). Based on Item~(1). it suffices to show that for any cycle $C$ with $|E(C)|\leq D$ we have
    \begin{align}{\label{eq-one-cycle}}
        |\Xi(C)| \leq k (1-\delta)^{ \frac{|E(C)|}{2} } n^{ -\frac{|E(C)|}{2} } \,.
    \end{align}
    Note that using \eqref{eq-def-Xi-intial} and \eqref{eq-def-Xi-itrative}, we have that
    \begin{align*}
        \Xi(C) = \tfrac{ \mathtt Q(S,\emptyset) }{ \mathtt P(S) } \cdot \big( \tfrac{\epsilon^2 \lambda}{n} \big)^{ |E(C)|/2  } = \tfrac{ (k-1) }{ \mathtt P(S) } \cdot (\epsilon^2 \lambda)^{ |E(C)|/2 } n^{ - |E(C)|/2 }  \,,
    \end{align*}
    where in the second equality we use Lemma~\ref{lem-expectation-over-chain}. Note that since $\langle \sigma_i, \sigma_j \rangle \in \{ -1,k-1 \}$, we have
    \begin{align*}
        \mathtt h(\sigma_i,\sigma_j) = a\omega(\sigma_i,\sigma_j) + b \,,
    \end{align*}
    where
    \begin{align*}
        a=[1+o(1)] \cdot \tfrac{ (k-1)\sqrt{1-\epsilon} + \sqrt{1+\epsilon(k-1)} }{ k } \mbox{ and } b= [1+o(1)] \cdot \tfrac{ \sqrt{1+\epsilon(k-1)} - \sqrt{1-\epsilon} }{ k } \,.
    \end{align*}
    Thus, using Lemma~\ref{lem-expectation-over-chain} we get that
    \begin{align*}
        \mathtt P(C) = a^{|E(C)|} + (k-1) b^{|E(C)|} \geq a^{|E(C)|} \,.
    \end{align*}
    Plugging this result into the bound on $\Xi(C)$ yields Item~(2) (recall \eqref{eq-choice-lambda-0}).

    Finally we prove Item~(3). Again based on Item~(1), we may assume that $S$ is connected. We prove Item~(3) by induction on $\tau(S)$. The case where $\tau(S)=0$ is proved in Item~(2). Now assume Item~(3) holds for $\tau(S) \leq m-1$, we now focus on $\tau(S)=m$. Using Lemma~\ref{lem-revised-decomposition-H-Subset-S}, we can decompose $S$ into $\mathtt t$ disjoint self-avoiding paths $P_{\mathtt 1}, \ldots, P_{\mathtt t}$ with $\mathtt t \leq 3m$ satisfying Lemma~\ref{lem-revised-decomposition-H-Subset-S}. In addition, applying Lemma~\ref{lem-revised-decomposition-H-Subset-S} for $H \subset S$ and $\mathcal L(H)=\mathcal L(S)=\emptyset$, we see that assuming $\tau(H)=m-r$ for some $r>0$, we then have $S \setminus H= P_{\mathtt i_1} \cup \ldots \cup \mathtt P_{\mathtt i_{s}}$ where $s\leq 3r$. Denote $\mathtt V=\cup_{1 \leq \mathtt i \leq m} \operatorname{EndP}(P_{\mathtt i})$, by conditioning on $\{ \sigma_u: u \in \mathtt V \}$ we see that
    \begin{align*}
        \Big\{ \prod_{(i,j)\in E(P_{\mathtt i})} \mathtt h(\sigma_i,\sigma_j) : 1 \leq \mathtt i \leq m \Big\}
    \end{align*}
    are conditional independent, with (denote $\operatorname{EndP}(P_{\mathtt i})=\{ u_{\mathtt i}, v_{\mathtt i} \}$)
    \begin{align*}
        \mathbb E\Big[ \prod_{(i,j)\in E(P_{\mathtt i})} \mathtt h(\sigma_i,\sigma_j) \mid \{ \sigma_u: u \in \mathtt V \} \Big] = a^{ |E(P_{\mathtt i})| } + b^{|E(P_{\mathtt i})|} \omega(u_{\mathtt i}, v_{\mathtt i}) \,.
    \end{align*}
    Thus, we have
    \begin{align*}
        \mathtt P(S) &= \mathbb E_{ \sigma_{\mathtt V} \sim \nu_{\mathtt V} } \mathbb E_{ \sigma_{\setminus \mathtt V} \sim \nu_{\setminus \mathtt V} } \Bigg[ \prod_{1 \leq \mathtt i \leq \mathtt t} \prod_{(i,j)\in E(P_{\mathtt i})} \mathtt h(\sigma_i,\sigma_j) \mid \{ \sigma_u: u \in \mathtt V \} \Bigg] \\
        &= \mathbb E_{ \sigma_{\mathtt V} \sim \nu_{\mathtt V} } \Bigg[ \prod_{1 \leq \mathtt i \leq \mathtt t} \prod_{ \{ u_{\mathtt i},v_{\mathtt i} \} = \operatorname{EndP}(P_{\mathtt i}) } \Big( a^{ |E(P_{\mathtt i})| } + b^{|E(P_{\mathtt i})|} \omega(u_{\mathtt i}, v_{\mathtt i}) \Big) \Bigg]  \,.
    \end{align*}
    Similarly we can show that (denote $\Lambda=\{ \mathtt i_1,\ldots,\mathtt i_s \}$)
    \begin{align*}
        \mathtt Q(S,H) &= \mathbb E_{ \sigma_{\mathtt V} \sim \nu_{\mathtt V} } \Bigg[ \prod_{\mathtt i \in [\mathtt t] \setminus \Lambda} \prod_{ \{ u_{\mathtt i},v_{\mathtt i} \} = \operatorname{EndP}(P_{\mathtt i}) } \Big( a^{ |E(P_{\mathtt i})| } + b^{|E(P_{\mathtt i})|} \omega(u_{\mathtt i}, v_{\mathtt i}) \Big) \prod_{\mathtt i \in \Lambda} \prod_{ \{ u_{\mathtt i},v_{\mathtt i} \} = \operatorname{EndP}(P_{\mathtt i}) } \omega(u_{\mathtt i}, v_{\mathtt i}) \Bigg] \\
        &\leq k^s \prod_{\mathtt i \in \Lambda} a^{-|E(P_{\mathtt i})|} \mathbb E_{ \sigma_{\mathtt V} \sim \nu_{\mathtt V} } \Bigg[ \prod_{1 \leq \mathtt i \leq \mathtt t} \prod_{ \{ u_{\mathtt i},v_{\mathtt i} \} = \operatorname{EndP}(P_{\mathtt i}) } \Big( a^{ |E(P_{\mathtt i})| } + b^{|E(P_{\mathtt i})|} \omega(u_{\mathtt i}, v_{\mathtt i}) \Big) \Bigg] \\
        &\leq k^{s} a^{ -(|E(S)|-|E(H)|) } \mathtt P(S) \,.
    \end{align*}
    Thus, we have that
    \begin{align*}
        \frac{ (\epsilon^2 \lambda)^{|E(S)|-|E(H)|} \mathtt Q(S,H) }{ \mathtt P(S) } &\leq (2k)^{s} (\epsilon^2 \lambda)^{|E(S)|-|E(H)|} a^{|E(S)|-|E(H)|} \\
        &\overset{\eqref{eq-choice-lambda-0}}{\leq} (2k)^{3r} (1-\delta/2)^{|E(S)|-|E(H)|} \,.
    \end{align*}
    Thus, we have (combining the induction hypothesis on $H$)
    \begin{align*}
        |\Xi(S)| &\leq (1-\delta/2)^{|E(S)|} n^{-|E(S)|/2} \sum_{ 1 \leq r \leq m } (2k)^{3r} (10(m-r))! \#\{ H \subset S: \mathcal L(H)=\emptyset \} \\
        &\leq (1-\delta/2)^{|E(S)|} n^{-|E(S)|/2} \sum_{ 1 \leq r \leq m } (2k)^{3r} (10(m-r))! D^{3r}  \\
        &\leq (10m)! (2kD)^{10m} (1-\delta/2)^{\frac{|E(S)|}{2}} n^{-\frac{|E(S)|}{2}} \,,
    \end{align*}
    where in the second inequality we use Lemma~\ref{lem-enu-Subset-small-graph} with $m_j=0$, leading to Item~(3).
\end{proof}

Now we prove Lemma~\ref{lem-bound-L2-norm} based on Lemma~\ref{lem-bound-Xi}.

\begin{proof}[Proof of Lemma~\ref{lem-bound-L2-norm}]
    It is easy to check that $u$ satisfies \eqref{eq-linear-equation-simplified}. We need to show that
    \begin{align}
        \sum_{ \sigma\in [k]^n } \sum_{ H \Subset \mathcal{K}_n, |E(H)| \leq D } u_{\sigma,H}^2 = \sum_{ \substack{ H \Subset \mathcal{K}_n, \mathcal L(H)=\emptyset \\ |E(H)|\leq D } } \Xi(H)^2 \label{eq-L2-norm-relax-1} 
    \end{align}
    is bounded by $O_{\delta}(1)$. Using Lemma~\ref{lem-bound-Xi}, we see that \eqref{eq-L2-norm-relax-1} is bounded by
    \begin{align}
        \sum_{ \substack{ m_3,\ldots,m_D \geq 0, t \geq 0 \\  v \geq 3m_3+\ldots+Dm_D } } 2^{10k+m_3+\ldots+m_D} k^{10t} n^{-(v+t)} (1-\delta/2)^{v+t} \cdot \mathsf{ENUM}'(m_3,\ldots,m_D;t) \,, \label{eq-L2-norm-relax-2}
    \end{align}
    where 
    \begin{align}{\label{eq-def-ENUM'}}
        \mathsf{ENUM}'(m_3,\ldots,m_D;t) = \#\Big\{ S \Subset \mathcal{K}_n: \mathcal L(S)=\emptyset, \mathfrak C_k(S)=m_k \mbox{ for } k \leq D, \tau(S)=t \Big\} \,.
    \end{align}
    Using Lemma~\ref{lem-enu-Subset-large-graph} with $H=\emptyset$, we see that
    \begin{align}{\label{eq-bound-ENUM'}}
        \mathsf{ENUM}'(m_3,\ldots,m_D;t) \leq n^v D^{10t} \prod_{3 \leq i \leq D} \frac{1}{m_i!} \,.
    \end{align}
    Plugging \eqref{eq-bound-ENUM'} into \eqref{eq-L2-norm-relax-2}, we get that
    \begin{align*}
        \eqref{eq-L2-norm-relax-2} &\leq \sum_{ \substack{ m_3,\ldots,m_D \geq 0, t \geq 0 \\ v \geq 3m_3+\ldots+Dm_D } } 2^{10t+m_3+\ldots+m_D} k^{10t} n^{-(v+t)} (1-\delta/2)^{v+t} \cdot n^v D^{10t} \prod_{3 \leq i \leq D} \frac{1}{i^{m_i}m_i!} \\
        &\leq \sum_{ \substack{ m_3,\ldots,m_D \geq 0, t \geq 0 \\ v \geq 3m_3+\ldots+Dm_D } } 2^{m_3+\ldots+m_D} (2^{10}D^{10}k/n)^{10t} (1-\delta/2)^{v} \prod_{3 \leq i \leq D} \frac{1}{m_i!} \\
        &\leq [1+o(1)] \sum_{ \substack{ m_3,\ldots,m_D \geq 0 \\ v \geq 3m_3+\ldots+Dm_D } } 2^{m_3+\ldots+m_D} (1-\delta/2)^{v} \prod_{3 \leq i \leq D} \frac{1}{i^{m_i}m_i!} \\
        & \leq O_{\delta,k}(1) \cdot \sum_{ m_3,\ldots,m_D \geq 0 } 2^{m_3+\ldots+m_D} \prod_{3 \leq i \leq D} \frac{1}{i^{m_i}m_i!} = O_{\delta,k}(1) \,.
    \end{align*}
    This concludes the desired result.
\end{proof}

\bibliographystyle{alpha}

\end{document}